\documentclass{article}

\usepackage{graphicx} 
\usepackage{amsmath,amsthm,amsfonts,amssymb}

\usepackage{natbib}

\usepackage{algorithm}
\usepackage{algorithmic}

\newcommand{\eat}[1]{}
\newcommand{\vpar}{\vspace*{.3em}}

\newcommand\stitle[1]{\vpar\noindent{\bf #1\/}}

\floatstyle{ruled}
\newfloat{algorithm}{tb}{loa}
\floatname{algorithm}{Algorithm}
\newtheorem{theorem}{Theorem}
\newtheorem{corollary}[theorem]{Corollary}

\newtheorem{lemma}[theorem]{Lemma}
\newtheorem{definition}[theorem]{Definition}

\newcommand{\abs}[1]{\lvert#1\rvert}
\newcommand{\defn}{:=}

\newcommand{\vnorm}[1]{\left\|#1\right\|}

\def\argmin{{\rm argmin}}

\def\var{{\rm Var}}
\def\E{{\ensuremath{\mathbb E}}}
\def\Prob{{\rm P}}

\long\def\comment#1{}

\newcommand{\R}{\ensuremath{\mathbb R}}

\newcommand{\sign}{\ensuremath{{\rm sign}}}

\newcommand{\M}{\mathcal{M}}

\newcommand{\acal}{\mathcal{A}}

\newcommand{\gcal}{\mathcal{G}}
\newcommand{\hcal}{\mathcal{H}}
\newcommand{\mcal}{\mathcal{M}}
\newcommand{\xcal}{\mathcal{X}}

\newcommand{\restrict}[1]{\ensuremath{\left.#1\right|}}
\newcommand{\ip}[1]{\ensuremath{\left\langle #1 \right\rangle}}

\newcommand{\td}{\tilde}

\renewcommand{\P}{\ensuremath{\mathbb P}}

\usepackage{bbm}
\newcommand{\Id}{\ensuremath{\mathbbm{1}}}

\begin{document} 

\title{An Analysis of the Convergence of Graph Laplacians}
\author{
        Daniel Ting \\
        Department of Statistics\\
        University of California, Berkeley\\
            \and
        Ling Huang\\
        Intel Research\\
        \and
        Michael Jordan \\
        Department of EECS and Statistics\\
        University of California, Berkeley\\
}





\maketitle

\begin{abstract}
Existing approaches to analyzing the asymptotics of 
graph Laplacians typically 
assume a well-behaved kernel function with smoothness
assumptions. We remove the smoothness assumption
and generalize the analysis of graph
Laplacians to include previously unstudied graphs
including kNN graphs.
We also introduce a kernel-free framework to analyze 
graph constructions with shrinking neighborhoods in general
and apply it to analyze locally linear embedding (LLE).
We also describe how for a given limiting Laplacian operator
desirable properties such as
a convergent spectrum and sparseness
can be achieved choosing the appropriate graph construction.

\end{abstract}

\section{Introduction}
Graph Laplacians have become a core technology throughout machine learning. 
In particular, they have appeared in clustering~\cite{Onclusterings, ConsistencySC},
dimensionality reduction~\cite{BelkinLapEigenmaps, Nadler06diffusionmaps}, 
and semi-supervised learning~\cite{SSLManifolds, SSLGaussian}.  

While graph Laplacians are but one member of a broad class of methods that 
use local neighborhood graphs to model data lying on a low-dimensional manifold 
embedded in a high-dimensional space, they are distinguished by their appealing 
mathematical properties, notably: (1) the graph Laplacian is the infinitesimal 
generator for a random walk on the graph, and (2) it is a discrete approximation 
to a weighted Laplace-Beltrami operator on a manifold, an operator which has 
numerous geometric properties and induces a smoothness functional.  These mathematical 
properties have served as a foundation for the development of a growing 
theoretical literature that has analyzed learning procedures based on the
graph Laplacian.  To review briefly, \citet{Bousquet03measurebased} proved 
an early result for the convergence of the unnormalized graph Laplacian 
to a regularization functional that depends on the squared density $p^2$. 
\citet{Belkin05FoundationLap} demonstrated the pointwise convergence of 
the empirical unnormalized Laplacian to the Laplace-Beltrami operator
on a compact manifold with uniform density.  \citet{LafonThesis} and 
\citet{Nadler06diffusionmaps} established a connection between graph 
Laplacians and the infinitesimal generator of a diffusion process. 
They further showed that one may use the degree operator to control
the effect of the density.
\citet{Hein05colt} 
combined and generalized these results for weak and pointwise (strong) 
convergence under weaker assumptions as well as providing rates for the 
unnormalized, normalized, and random walk Laplacians. They also make explicit
the connections to the weighted Laplace-Beltrami operator.
 \citet{singer06LapConvRate} 
obtained improved convergence rates for a uniform density.  \citet{gine05laplacian} 
established a uniform convergence result and 
functional central limit theorem to extend the pointwise convergence results.
\citet{ConsistencySC} and \citet{BelkinConvLapEigenmaps} presented spectral 
convergence results for the eigenvectors of graph Laplacians in the fixed 
and shrinking bandwidth cases respectively.  

Although this burgeoning literature has provided many useful insights, 
several gaps remain between theory and practice.  Most notably, in constructing 
the neighborhood graphs underlying the graph Laplacian, several choices
must be made, including the choice of algorithm for constructing the 
graph, with $k$-nearest-neighbor (kNN) and kernel functions providing the main
alternatives, as well as the choice of parameters 
($k$, kernel bandwidth, normalization weights).  
These choices can lead to the graph Laplacian generating fundamentally 
different random walks and approximating different weighted Laplace-Beltrami operators. 
The existing theory has focused on one specific choice in which graphs are
generated with smooth kernels with shrinking bandwidths.  But a variety of
other choices are often made in practice, including kNN graphs, $r$-neighborhood 
graphs, and the ``self-tuning'' graphs of \citet{ZelnikSelfTuningSpectral}. 
Surprisingly, few of the existing  convergence results apply to these
choices (see \citet{MaierNCutConvergence} for an exception).

This paper provides a general theoretical framework for analyzing 
graph Laplacians and operators that behave like Laplacians.  
Our point of view differs from that found in the existing literature; specifically, 
our point of departure is a stochastic process framework that utilizes the
characterization of diffusion processes via drift and diffusion terms.
This 
yields a general kernel-free framework for analyzing graph Laplacians with shrinking 
neighborhoods.  We use it to extend the pointwise results of 
\citet{HeinGraphNormalizations} to cover non-smooth kernels and introduce location-dependent 
bandwidths.  Applying these tools we are able to identify the asymptotic limit for a 
variety of graphs constructions including kNN, $r$-neighborhood, and ``self-tuning'' 
graphs.  We are also able to provide an analysis for Locally Linear Embedding
\citep{roweis00LLE}.

A practical motivation for our interest in graph Laplacians based on kNN graphs
is that these can be significantly sparser than those constructed using kernels,
even if they have the same limit.  Our framework allows us to establish this
limiting equivalence.  On the other hand, we can also exhibit cases in which 
kNN graphs converge to a different limit than graphs constructed from kernels, 
and that this explains some cases where kNN graphs perform poorly.  Moreover,
our framework allows us to generate new algorithms: in particular, by using
location-dependent bandwidths we obtain a class of operators that have nice 
spectral convergence properties that parallel those of the normalized Laplacian 
in \citet{ConsistencySC}, but which converge to a different class of limits.

\section{The Framework}

Our work exploits the connections among 
diffusion processes, elliptic operators (in particular the weighted
Laplace-Beltrami operator), 
and stochastic differential equations (SDEs). This builds upon 
the diffusion process viewpoint in \citet{Nadler06diffusionmaps}. 
Critically, we make the connection to the drift and diffusion terms
of a diffusion process. This allows us to present a kernel-free framework
for analysis of graph Laplacians as well as giving a better
intuitive understanding of the limit diffusion process.

We first give a brief overview of these connections and 
present our general framework
for the asymptotic analysis of graph Laplacians
as well as providing some relevant 
background material.
We then introduce 
our assumptions and derive our main results for the 
limit operator for a wide range
of graph construction methods.
We use these to calculate
 asymptotic limits for specific graph constructions.


\subsection{Relevant Differential Geometry}
\label{sec:diff geom}
Assume $\M$ is a $m$-dimensional manifold embedded in $\R^b$.
To identify the asymptotic infinitesimal generator of a diffusion
on this manifold, we will derive the drift and 
diffusion terms in normal coordinates at each point.
We refer the reader to \citet{boothby} for an exact definition of 
normal coordinates.  For our purposes it suffices to note that
normal coordinates are coordinates in $\R^m$ that behave 
roughly as if the neighborhood was projected onto the tangent 
plane at $x$. 
The extrinsic coordinates are the coordinates 
$\R^b$ in which the manifold is embedded.  
Since the density, 
and hence integration, is defined with respect to the manifold, 
we must relate to link normal coordinates 
$s$ around a point $x$ with the extrinsic coordinates $y$. 
This relation may be given as follows:

\begin{align}
\label{eqn:embedCoord}
y-x = H_xs + L_x(ss^T) + O(\vnorm{s^3}),
\end{align} 
where 
$H_x$ is a linear isomorphism between 
the normal coordinates in $R^m$ and the 
$m$-dimensional tangent plane $T_x$ at $x$.
$L_x$ is a linear operator describing the curvature of the manifold
and takes $m \times m$ positive semidefinite matrices
into the space orthogonal to the tangent plane, $T_x^{\perp}$.
More advanced readers will note
that this statement is Gauss' lemma 
and $H_x$ and $L_x$ are related to the first and second fundamental forms.

We are most interested in limits involving the weighted 
Laplace-Beltrami operator, a particular
second-order differential operator.
\subsection{Weighted Laplace-Beltrami operator}
\begin{definition}[Weighted Laplace-Beltrami operator]
\label{def:Laplace-Beltrami}
The weighted Laplace-Beltrami operator with respect to the density $q$
is the second-order differential operator defined by
$\Delta_{q} \defn \Delta_{\M} - \frac{\nabla q^T}{q} \nabla$
where $\Delta_{\M} \defn div \circ \nabla$ 
is the unweighted Laplace-Beltrami operator. 
\end{definition}
It is of particular interest
since it induces a smoothing functional for $f \in C^2(\M)$ with support
contained in the interior of the manifold:
\begin{align}
 \langle f , \Delta_{q} f \rangle_{L(q)} = 
\vnorm{ \nabla f}_{L_2(q)}^2.
\end{align} 
Note that existing literature on asymptotics
of graph Laplacians often refers to the $s^{th}$
weighted Laplace-Beltrami operator as $\Delta_{s}$ where $s \in \R$.
This is $\Delta_{p^s}$ in our notation. 
For more information on the weighted Laplace-Beltrami operator see 
\citet{grigor2006heat}.

\subsection{Equivalence of Limiting Characterizations}

We now establish the promised connections among
elliptic operators, diffusions, SDEs, and graph Laplacians.
We first show that elliptic operators define 
diffusion processes and SDEs and vice versa.  An elliptic 
operator $\gcal$ is a second order differential operator of the form
\[
\gcal f(x) = \sum_{ij} a_{ij}(x) 
\frac{\partial^2f(x)}{\partial x_i \partial x_j} 
+ \sum_i b_i(x) \frac{\partial f(x)}{\partial x_i } + c(x) f(x),
\]
where the $m \times m$ coefficient matrix $(a_{ij}(x))$ is positive 
semidefinite
for all $x$.
If we use normal coordinates for a manifold, we see that
the weighted Laplace-Beltrami operator $\Delta_q$ is a special case
of an elliptic operator with $(a_{ij}(x)) = I$, the identity matrix,
$b(x) = \frac{\nabla q(x)}{q(x)}$,
and $c(x) = 0$. 
Diffusion processes are related via a result by Dynkin
which states that given a diffusion process, the generator of the process
is an elliptic operator. 

The (infinitesimal) generator $\gcal$ of a diffusion process $X_t$
is defined as
\begin{align*}
\gcal f(x) \defn \lim_{t \to 0} \frac{\E_x f(X_t) - f(x)}{t}
\end{align*}
when the limit exists and convergence is uniform over $x$.
Here $\E_x f(X_t) = \E(f(X_t) | X_0 = x)$.
A converse relation holds as well. 
The Hille-Yosida theorem characterizes when a linear operator,
such as an elliptic operator, is the generator of a stochastic process.
We refer the reader to \cite{Kallenberg} for proofs.

A time-homogeneous stochastic differential equation (SDE) defines a 
diffusion process as a solution (when one exists) to the equation
\[
dX_t = \mu(X_t)dt + \sigma(X_t) dW_t,
\]
where $X_t$ is a diffusion process taking values in $\R^d$.
The terms $\mu(x)$ and $\sigma(x)\sigma(x)^T$ are the 
{\em drift} and {\em diffusion} terms of the process.

By Dynkin's result, the generator $\gcal$ of this process defines
an elliptic operator and a simple calculation shows the operator is
\[
\gcal f(x) = \frac{1}{2} \sum_{ij} \left(\sigma(x)\sigma(x)^T\right)_{ij}
\frac{\partial^2 f(x)}{\partial x_i \partial x_j}
+ \sum_i \mu_i(x) \frac{\partial f(x)}{\partial x_i }.
\]
In such diffusion processes there is no absorbing state and the term
in the elliptic operator $c(x) = 0$.
We note that one may also consider more general diffusion processes 
where $c(x) \leq 0$. When $c(x) < 0$
then we have the generator of a diffusion process with killing where $c(x)$ 
determines the killing rate of the diffusion at $x$.

To summarize, 
we see that a SDE or diffusion process define
an elliptic operator, and importantly, the coefficients are the
drift and diffusion terms, and the reverse relationship
holds: An elliptic operator defines a diffusion 
under some regularity conditions on the coefficients.



All that remains then is to connect diffusion processes in continuous space
to graph Laplacians on a finite set of points. Diffusion approximation
theorems provide this connection.
We state one version of such a theorem
.

\begin{theorem}[Diffusion Approximation]
\label{thm:DiffusionApprox}
Let $\mu(x)$ and $\sigma(x)\sigma(x)^T$ be drift and diffusion terms
for a diffusion process
defined on a compact set $S \subset \R^b$, and let
and $G$ be the corresponding infinitesimal generator.
Let $\{Y_t^{(n)}\}_t$ be Markov chains with 
transition matrices $P_n$ on state spaces $\{x_i\}_{i=1}^n$
for all $n$, and let $c_n > 0$ define a sequence of scalings. 
Put
\begin{alignat*}{3}
&\hat{\mu}_n(x_i) &=&  c_n \E(Y_1^{(n)} -x_i | Y_0^{(n)} = x_i) \\
\hat{\sigma}_n(x_i)&\hat{\sigma}_n(x_i)^T &=&  c_n \var(Y_1^{(n)}| Y_0^{(n)} = x_i).
\end{alignat*}
Let  $f \in C^2(S)$. If for all $\epsilon > 0$
\begin{gather*}
\hat{\mu}_n(x_i) \to \mu(x_i), \\
\hat{\sigma}_n(x_i)\hat{\sigma}_n(x_i)^T \to \sigma(x_i)\sigma(x_i)^T,  \\
c_n \sup_{i \leq n}  \Prob\left( \left. \vnorm{Y_1^{(n)} -x_i} > \epsilon \right| Y_0^{(n)} = x_i\right) \to 0,
\end{gather*}
then the  generators $A_n f = c_n(P_n-I) f \to Gf$ 
Furthermore, for any bounded $f$ and $t_0 > 0$ and
the continuous-time transition kernels $T_n(t) = exp(tA_n)$
and $T$ the transition kernel for $G$, we have
$T_{n}(t) f \to T(t)f$ uniformly in $t$ for $t < t_0$. 
\end{theorem}

\begin{proof}

We first examine the case when $f(x) = x$.
By assumption,
\begin{align*}
A_n \pi_n x &= c_n(P_n -I)x = c_n \E(Y_1^{(n)} -x_i | Y_0^{(n)} = x_i) \\
&= \mu_n(x) \to \mu(x) = A x.
\end{align*}
Similarly if $f(x) = x x^T$, 
$\vnorm{ A_n \pi_n f - A f}_\infty \to 0$.
If $f(x)=1$, then $A_n \pi_n f = \pi_n A f = 0$.
Thus, by linearity of $A_n$,
 $A_n \pi_n f \to A f$ for any quadratic polynomial $f$.

Taylor expand $f$ to obtain
$f(x + h) = q_x(h) + \delta_x(h)$ where $q_x(h)$ is a quadratic 
polynomial in $h$.
Since the second derivative
is continuous and the support of $f$ is compact,
$\sup_{x \in \M} \delta_x(h) =o( \vnorm{h}^2 )$
and $\sup_{x,h} \delta_x(h) < M$ for some constant $M$.

Let $\Delta_n = Y_1^{(n)} -x_i$. We may bound $A_n$ acting
on the remainder term $\delta_x(h)$ by
\begin{align*}
\sup_x A_n \delta_x &= c_n \E( \delta_x( \Delta_n ) | Y_0^{(n)} = x) \\
&\leq \sup_x c_n \E( \delta_x( \Delta_n) \Id(\vnorm{\Delta_n} \leq \epsilon) 
| Y_0^{(n)} = x) + \\
& \qquad \qquad M \sup_x c_n  \P( \vnorm{\Delta_n} > \epsilon | Y_0^{(n)} = x) \\
&= o( c_n \E( \vnorm{\Delta_n}^2| Y_0^{(n)} = x)  ) + M \sup_x c_n\P( \vnorm{\Delta_n} > \epsilon | Y_0^{(n)} = x) \\
&= o( 1 )
\end{align*}
where the last equality holds by the assumptions on the uniform convergence
of the diffusion term $\hat{\sigma}_n \hat{\sigma}_n^T$ and 
on the shrinking jumpsizes.

Thus, $A_n \pi_n f \to A f$ for any $f \in C^2(\M)$.

The class of functions $C^2(\M)$ is dense in $L_{\infty}(\M)$
and form a core for the generator $A$.
Standard theorems give equivalence between
strong convergence of infinitesimal generators on a core and
uniform strong convergence of transition kernels on a Banach space
(e.g. Theorem 1.6.1 in ~\cite{EthierKurtz}).
\end{proof}

We remark that though the results we have discussed thus far are stated in the context of the 
extrinsic coordinates $\R^b$, 
we describe appropriate extensions in terms of normal coordinates 
in the appendix. 


\subsection{Assumptions}
\label{sec:assumptions}
%
%
We describe here the assumptions and notation for the
rest of the paper.
The following assumptions we will refer to as the {\em standard assumptions}.

Unless stated explicitly otherwise, 
let $f$ be an arbitrary function in $C^2(\M)$.

\stitle{Manifold assumptions.} Assume $\M$ us a smooth $m$-dimensional manifold 
isometrically embedded in $\R^b$ via the map $i: \M \to \R^b$. 
The essential conditions that we require on the manifold are
\begin{enumerate}
\item Smoothness, the map $i$ is a smooth embedding. 
\item A single radius $h_0$ such that for all $x \in supp(f)$,
 $\M \cap B(x,h_0)$ is a neighborhood of $x$ with normal coordinates, and
\item Bounded curvature of the manifold over $supp(f)$, i.e.
that the second fundamental form is bounded .
\end{enumerate}
When the manifold is smooth and compact, then these conditions are satisfied.

Assume points $\{x_i\}_{i=1}^\infty$ are sampled i.i.d. from a density
$p \in C^2(\M)$ with respect to the natural volume element of the manifold, 
and that $p$ is bounded away from 0.

\stitle{Notation.} For 
brevity, we will always use $x,y \in \R^b$ to be points on $\M$ expressed
in extrinsic coordinates and $s \in \R^m$ to be normal coordinates
for $y$ in a neighborhood centered at $x$. 
Since they represent the same point, 
we will also use $y$ and $s$ interchangeably as function arguments, 
i.e. $f(y) = f(s)$. Whenever we take a gradient,
it is with respect to normal coordinates.


\stitle{Generalized kernel.}
Though we use a kernel free framework,
our main theorem utilizes a kernel, but one that
is generalizes previously studied kernels 
by 1) considering non-smooth base kernels $K_0$,
2) introducing location dependent bandwidth functions $r_x(y)$, 
and 3) considering general weight functions $w_x(y)$. 
Our main result also handles 4) random weight and bandwidth functions.

Given a
bandwidth scaling parameter $h > 0$, 
define a new kernel by
\begin{align}
\label{eqn:kernel def}
K(x,y) = w_x(y) K_0 \left(\frac{\vnorm{y-x}}
{h r_x(y)}\right).
\end{align}

Previously analyzed constructions for smooth kernels with compact support
are described by this more general kernel with $r_x = 1$
and $w_x(y) = d(x)^{-\lambda}d(y)^{-\lambda}$ 
where $d(x)$ is the degree function and
$\lambda \in \R$ is some constant.

The directed kNN graph is obtained if $K_0(x,y) = \Id( \vnorm{x-y} \leq 1)$, 
$r_x(y) = $ distance to the $k^{th}$ nearest neighbor of $x$, 
and $w_x(y) = 1$ for all $x,y$. 

We note that the kernel $K$ is not necessarily symmetric; however, 
if $r_x(y) = r_y(x)$ and $w_x(y) = w_y(x)$ for all 
$x,y \in \M$ then the kernel is symmetric and the corresponding 
unnormalized Laplacian is positive semi-definite.

\stitle{Kernel assumptions.} 
We now introduce our assumptions on the choices $K_0, h, w_x, r_x$
that govern the graph construction.
Assume that the base kernel $K_0 : \R_+ \to \R_+$ 
has bounded variation and compact support
and $h_n > 0$ form a sequence of bandwidth scalings.
For (possible random) location dependent bandwidth and weight functions
 $r_x^{(n)}(\cdot) > 0, w_x^{(n)}(\cdot) \geq 0$,
assume that they converge to $r_x(\cdot), w_x(\cdot)$ respectively 
and the convergence is uniform over $x \in \M$. Further assume
they have
Taylor-like expansions for all $x,y \in \M$ with $\vnorm{x-y} < h_n$
\begin{align}
\label{eqn:bw func}
\begin{split}
r_x^{(n)}(y) &= r_x(x) + (\dot{r}_x(x) + \alpha_x \sign(u_x^Ts) u_x)^Ts 
+ \epsilon_r^{(n)}(x,s) \\
w_x^{(n)}(y) &= w_x(x) + \nabla w_x(x)^T s + \epsilon_w^{(n)}(x,s) 
\end{split}
\end{align}
where the approximation error is uniformly bounded by 
 \begin{align*}
 \sup_{x\in \M, \vnorm{s} < h_n} |\epsilon_r^{(n)}(x,s)| &= O(h_n^2) \\ 
 \sup_{x\in \M, \vnorm{s} < h_n} |\epsilon_w^{(n)}(x,s)| &= O(h_n^2) 
 \end{align*}

We briefly motivate the choice of assumptions.
The bounded variation condition allows for non-smooth base kernels but 
enough regularity to obtain limits.
The Taylor-like expansions allow give conditions where the limit
is tractable to analytically compute as well as allowing for randomness
in the remainder term as long as it is of the correct order.
The particular expansion for the location dependent bandwidth
allows one to analyze undirected kNN graphs, which exhibit 
a non-differentiable location dependent bandwidth (see section 
\ref{sec:undirected knn}).
Note that we do not constrain
the general weight functions $w_x^{(n)}(y)$ to be a power
of the degree function, $d_n(x)^{\alpha} d_n(y)^{\alpha}$
nor impose a particular functional form for 
location dependent bandwidths $r_x$. 
This gives us two degrees
of freedom, which allows the same asymptotic limit
be obtained for an entire class of parameters 
governing the graph construction. 
In section \ref{sec:construction reasons}, 
we discuss one may choose a graph construction that
has more attractive finite sample properties than other constructions
that have the same limit.

\stitle{Functions and convergence.} We define here what we mean by convergence
when the domains of the functions are changing.
When take $g_n \to g$ 
where $domain(g_n) = \xcal_n \subset \M$, 
to mean
$\vnorm{g_n - \pi_n g }_{\infty} \to 0$
where $\pi_n g = \restrict{g}_{\xcal_n}$ 
is the restriction of $g$ to $\xcal_n$. 
Likewise, 
for operators $T_n$ on functions with domain $\xcal_n$, 
we take $T_n g = T_n \pi_n g$. 
Convergence of operators 
$T_n \to T$ means $T_nf \to Tf$ for all $f \in C^2(\M)$.
When $\xcal_n = \M$ for all $n$, 
this is convergence in the strong operator topology
under the $L_{\infty}$ norm.

We consider the limit of the random walk Laplacian defined by
as $L_{rw} = I - D^{-1} W$ where $I$ is the identity, $W$ is the matrix of edge weights,
and $D$ is the diagonal degree matrix.

\subsection{Main Theorem}
Our main result is stated in the following theorem.
\begin{theorem}
\label{thm:mainTheorem}
Assume the standard assumptions hold eventually with probability 1. 
If the bandwidth scalings $h_n$
satisfy $h_n \downarrow 0$ and $nh_n^{m+2} / \log n \to \infty$,
then for graphs constructed using
the kernels
\begin{align}
\label{eqn:kernel}
K_n(x,y) = {w_x^{(n)}(y)} K_0\left(\frac{\vnorm{y-x}}
{h_n r_x^{(n)}(y)}\right)
\end{align}
there exists a 
constant $Z_{K_0,m} > 0$ depending only on the base kernel
$K_0$ and the dimension $m$ such that for 
$c_n = Z_{K_0,m} / h^{2}$,
\[
-c_n L_{rw}^{(n)}f \to Af
\]
where $A$ is
the infinitesimal generator of a diffusion process with
the following drift and diffusion terms given in normal coordinates:
\begin{align*}
&\mu_s(x) = 
r_x(x)^{2} \left( \frac{\nabla p(x)}{p(x)}
+ \frac{\nabla w(x)}{w(x)} + (m+2)\frac{\dot{r}_x(x)}{r_x(x)} \right),& \\
&\sigma_s(x)\sigma_s(x)^T = r_x(x)^{2} I&
\end{align*}
where $I$ is the $m \times m$ identity matrix.

%
\end{theorem}


%
\begin{proof}
We apply the diffusion approximation theorem 
(Theorem \ref{thm:DiffusionApprox}) to obtain convergence of the 
random walk Laplacians.
Since $h_n \downarrow 0$, 
the probability of a jump of size $> \epsilon$ equals 0 eventually.
Thus, we simply need to 
show uniform convergence of the drift and diffusion terms
and identify their limits. 
We leave the detailed calculations in the appendix and present the main ideas
in the proof here.

We first assume that $K_0$ is an indicator kernel.
To generalize, we note that for kernels 
of bounded variation,
we may write
$K_0(x) = \int \Id(|x| < z) d\eta_{+}(z) - \int \Id(|x| < z) d\eta_{-}(z)$
for some finite positive measures $\eta_{-}, \eta_{+}$ with compact support. 
The result for general kernels then follows from Fubini's theorem.

We also initially assume that we are given the true density $p$.
After identifying the desired limits given the true density,
we show that the empirical version converges uniformly to the correct quantities.


The key calculation is lemma \ref{lemma:IdLemma} in the appendix which
establishes that
integrating against an indicator kernel is like integrating over a sphere
re-centered on $h_n^2 \dot{r}_x(x)$. 

Given this calculation
 and by Taylor expanding the non-kernel terms, one obtains the
infinitesimal first and second moments and the degree operator.
\begin{align*}
\begin{split}
M_1^{(n)}(x) 
&= \frac{1}{h_n^m} \int s K_n(x,y) p(y) ds \\
&= \frac{1}{h_n^m} \int s w_x^{(n)}(s) K_0\left(\frac{\vnorm{y-x}}{h_n r_x^{(n)}(s)}\right) p(s)ds\\
&= \frac{1}{h_n^m} \int s \left(w_x(x) + \nabla w_x(x)^T s + O(h_n^2)\right) 
\left(p(x) + \nabla p(x)^T s + O(h_n^2) \right) \times \\
& \quad 
\times K_0\left(\frac{\vnorm{y-x}}{h_n r_x^{(n)}(s)}\right)ds\\
&=  C_{K_0,m} h_n^{2} r_x(x)^{m+2} \left( w_x(x) \frac{\nabla p(x)}{m+2} 
 + p(x) \frac{\nabla w_x(x) }{m+2}
+ w_x(x) p(x) \dot{r}_x(x) + o(1) \right) 
\end{split}
\end{align*}
\begin{align*}
M_2^{(n)}(x) &= \frac{1}{h_n^m} \int ss^T K_n(x,y) p(y)ds \\
&= \frac{1}{h_n^m}  \int ss^T w_x^{(n)}(s) K_0\left(\frac{\vnorm{y-x}}{h_n r_x^{(n)}(s)}\right) p(s)ds\\
&= \frac{1}{h_n^m}  \int ss^T \left(w_x(x) + O(h_n)\right) \left(p(x) +  O(h_n) \right) 
 K_0\left(\frac{\vnorm{y-x}}{h_n r_x^{(n)}(s)}\right)ds\\
&= \frac{C_{K_0,m}}{m+2} h_n^{2} r_x(x)^{m+2} \left(w_x(x)p(x) I + O(h_n)\right), 
\end{align*}
\begin{align}
\label{eqn:degop}
d_n(x) &= \frac{1}{h_n^m} \int K_n(x,y) p(y) ds\\
&= \frac{1}{h^m} \int w_x^{(n)}(s) K_0\left(\frac{\vnorm{y-x}}{h_n r_x^{(n)}(s)}\right) p(s)ds\\
&= \frac{1}{h^m} \int \left(w_x(x) + O(h_n)\right) \left(p(x) +  O(h_n) \right)
K_0\left(\frac{\vnorm{y-x}}{h_n r_x^{(n)}(s)}\right) ds\\
&= C_{K_0,m}' r_x(x)^m \left(w_x(x)p(x) + O(h_n)\right)
\end{align}
where $C_{K_0,m} = \int u^{m+2} d\eta$, $C_{K_0,m}' = \int u^{m} d\eta$
and $\eta$ is the signed measure $\eta = \eta_+ - \eta_-$.

%
%
Let $Z_{K_0,m} = (m+2) \frac{C_{K_0,m}'}{C_{K_0,m}}$ and $c_n = Z_{K_0,m}/h_n^2$.
Since $K_n/d_n$ define Markov transition kernels, taking the limits 
$\displaystyle \mu_s(x) = \lim_{n\to \infty} c_n M_1^{(n)}(x)/ d_n(x)$ 
and $\displaystyle \sigma_s(x)\sigma_s(x)^T = \lim_{n\to \infty} c_n M_2^{(n)}(x)/ d_n(x)$
and applying the diffusion approximation theorem 
gives the stated result.

To more formally apply the diffusion approximation theorem we may calculate
the drift and diffusion in extrinsic coordinates. In extrinsic coordinates,
we have
\begin{align*}
&\mu(x) = r_x(x)^{2} H_x \Big( \frac{\nabla p(x)}{p(x)}
+ \frac{\nabla w_x(x)}{w_x(x)} +
(m+2)\frac{\dot{r}_x(x)}{r_x(x)} \Big)&
\\ &\qquad \qquad \qquad \qquad \qquad \qquad \qquad \qquad \quad
+ r_x(x)^2 L_x(I),&\\ 
&\sigma(x)\sigma(x)^T = r(x)^2\Pi_{T_x},&
\end{align*}
where $\Pi_{T_x}$ is the projection onto the tangent plane at $x$,
and $H_x$ and $L_x$ are 
the linear mappings between normal coordinates and extrinsic coordinates 
defined in Eqn~\eqref{eqn:embedCoord}.

 We now consider the convergence of the empirical quantities.
 For non-random $r_x^{(n)} = r_x, w_x^{(n)} = w_x$,
 the uniform and almost sure convergence of the empirical quantities to
 the true expectation follows from an application of Bernstein's inequality.
 In particular, 
 the value of 
 $F_n(x, S) = S_i K\left( \frac{\vnorm{Y-x}}{h_n r_x(Y)} \right)$
 is bounded by $K_{max} h_n$, 
 where $S$ is $Y$ in normal coordinates and $K_{max}$ depends on the kernel
 and the maximum curvature of the manifold.
 Furthermore, the second moment calculation for $M_2^{(n)}$ gives
 that the variance $\var( F_n(x, S) )$ is bounded by $c h_n^{m+2}$
 for some constant $c$ that depends on $K$ and the max of $p$, and 
 does not depend on $x$.
 By Bernstein's inequality
 and a union bound, we have
 \begin{align}
\nonumber  &Pr\left( \sup_{i \leq n} \left| \E_n \frac{1}{h_n^{m+2}} F_n(x_i,Y) - 
 \frac{1}{h_n^{2}} M_1^{(n)} \right| > \epsilon \right) \\
\nonumber &= Pr\left( \sup_{i \leq n} \left| \E_n F_n(x_i,Y) - 
 \E F_n(x_i, Y)\right| > \epsilon h_n^{m+2}\right) \\
 \label{eqn:bernstein bound}
 & < 2 n 
 \exp\left( -\frac{\epsilon^2 }{2c/(nh_n^{m+2}) + 2K_{max}\epsilon / (3nh_n^{m+1}) }\right). 
 \end{align}
 The uniform convergence a.s. of the first moment follows from Borel-Cantelli. 
 Similar inequalities are attained for the empirical second moment and degree terms.

Now assume $r_x^{(n)},w_x^{(n)}$ are random and define $F_n$ as before.
To handle the random weight and bandwidth function case,
we first choose deterministic weight and bandwidth functions to maximize
the first moment under a constraint that is satisfied eventually a.s..
Define
\begin{align*}
\overline{w}_x^{(n)}(y) &= w_x(y) + \kappa h_n^2 sign(s_i) \\ 
\overline{r}_x^{(n)}(y) &= r_x(x) + (\dot{r}_x(x) + \alpha_x \sign(u_x^Ts) u_x)^Ts  - \kappa h_n^2 sign(s_i) \\ 
\overline{F}_{n}(y) &= s_i \overline{w}_x^{(n)}(y) K_0 \left(\frac{\vnorm{y-x}}{h_n \overline{r}_x^{(n)}(y)} \right)
\end{align*}
for some constant $\kappa$ such that 
$\overline{r}_x^{(n)} < r_x^{(n)}$ and 
$\overline{w}_x^{(n)} > w_x^{(n)}$ eventually.
This is possible since the perturbation terms 
$\epsilon_r^{(n)}(x,s), \epsilon_w^{(n)}(x,s)= O(h_n^2)$.
Thus, we have $\overline{F}_{\kappa,n}(x,y) > F_n(x,y)$ for all 
$x,y \in \M$ eventually with probability 1.
Since $\overline{F}_{\kappa,n}(x,Y)$ uses deterministic weight and bandwidth functions, we obtain i.i.d. random variables and
may apply the Bernstein bound on $\overline{F}_{\kappa,n}(x,y)$ to 
obtain an upper bound on the empirical quantities, namely
$\E_n \overline{F}_{\kappa,n}(x,Y) > \E_n F_{n}(x,Y)$ for all $x \in \M$ 
eventually with probability 1.
We may similarly obtain a lower bound. 
By lemma \ref{lemma:shift sphere},
the difference between the expectation of the
upper bound and the  is
$\E \overline{F}_{\kappa,n}(x,Y) - \E \overline{F}_{0,n}(x,Y)= o(\kappa h_n^{m+2})$.
Applying the squeeze theorem
gives a.s. uniform convergence of the empirical first moment 
$M_1^{(n)}/h_n^{2}$. 
The degree and second moment terms are handled similarly.

Since $p,w_x,r_x$ are all assumed to be bounded away from $0$, 
the scaled degree operators $d_n$ are eventually bounded away from 0 
with probability 1, and the
continuous mapping theorem applied to $\frac{M_i^{(n)}/h_n^{2}}{d_n}$
gives a.s. uniform convergence of the drift and diffusion.

\end{proof}

\subsection{Unnormalized and Normalized Laplacians}
\label{sec:UnnormNormLap}
While our results are 
for the infinitesimal generator of a diffusion process, 
that is, 
for the limit of the random walk Laplacian $L_{rw} = I - D^{-1}W$,
it is easy to generalize them to 
the unnormalized Laplacian $L_{u} = D-W = D L_{rw}$
and symmetrically normalized Laplacian $L_{norm} = I - D^{-1/2}WD^{-1/2} = 
D^{1/2} L_{rw} D^{-1/2}$. 


\begin{corollary}
\label{corr:Unnorm and norm Lap}
Take the assumptions in Theorem \ref{thm:mainTheorem}, and
let $A$ be the limiting operator of the random walk Laplacian.
The degree terms 
$d_n(\cdot)$ converge uniformly a.s.
to a function $d(\cdot)$, and
\begin{align*}
-c_n' L_{u}^{(n)}f \to d \cdot A f \quad \mbox{a.s.}
\end{align*}
where $c_n' = c_n / h^m$.
Furthermore, under the additional assumptions
$nh_n^{m+4} / \log n \to \infty$,
$\sup_{x,y} |w^{(n)}_x - w_x| = o(h_n^2)$,
$\sup_{x,y} |r^{(n)}_x - r_x| = o(h_n^2)$,
and $d,w_x,r_x \in C^2(\M)$, 
we have
\begin{align*}
-c_n L_{norm}^{(n)}f \to d^{1/2} \cdot A (d^{-1/2}f) \quad \mbox{a.s.}
\end{align*}
\end{corollary}

\begin{proof}


For any two functions $\phi_1,\phi_2: \mcal \to \R$, 
define $g_u(\phi_1,\phi_2) = (\phi_1(\cdot), f_1(\cdot) \phi_2(\cdot))$.
We note that $g_u$ is a continuous mapping 
in the $L_{\infty}$ topology and 
\[
(d_n, c_n' L_{u}^{n}f) = g_u(d_n, c_n L_{rw}f).
\]
By the continuous mapping theorem,
if $d_n \to d$ a.s. and $c_n L_{rw}^{(n)}f \to Lf$ a.s. in the 
then 
\[
c_n' L_{u}^{(n)} \to d \cdot Lf.
\]
Thus, convergence of the random walk Laplacians implies convergence of the unnormalized
Laplacian under the very weak condition of convergence of the degree operator to a bounded function. 

Convergence of the normalized Laplacian is slightly trickier.
We may write the normalized Laplacian as 
\begin{align}
L_{norm}^{(n)} f &= d_n^{1/2} L_{rw}^{(n)} (d_n^{-1/2} f) \\
&= d_n^{1/2} L_{rw}^{(n)} (d^{-1/2}f) + d_n^{1/2} L_{rw}^{(n)} (d_n^{-1/2}-d^{-1/2}) f).
\end{align}
Using the continuous mapping theorem, 
we see that convergence of the normalized Laplacian,
$c_n L_{norm}^{(n)}f \to d^{-1/2} L_{rw} (d^{-1/2}f)$, is equivalent to showing
$c_n L_{rw}^{(n)} ((d_n^{-1/2}-d^{-1/2}) f) \to 0$.
A Taylor expansion of the inverse square root gives that showing
$c_n L_{rw}^{(n)} (d_n-d) \to 0$ is sufficient to prove convergence. 

We now verify conditions which will ensure that 
the degree operators will converge at the appropriate rate.
We further decompose the empirical degree operator into the
bias $\E d_n - d$ and empirical error $d_n - \E d_n$.

Simply carrying out the Taylor expansions to higher order terms in 
the calculation of the degree function $d_n$ in Eq.~\ref{eqn:degop},
and using the refined calculation of the zeroth moment in 
lemma~\ref{lemma:zeroth moment} in the appendix,
the bias of the degree operator is 
$d_n - d = h_n^2 b + o(h_n^2)$ for some uniformly bounded, continuous function $b$.


Thus we have,
\begin{align}
c_n L_{rw}^{(n)} (d_n-d) &= c_n h_n^2 \vnorm{(I - P_n) b}_{\infty} + o(1) = o(1) 
\end{align}
since $c_n h_n^2$ is constant and $\vnorm{(I - P_n) \phi}_{\infty} \to 0$ for any
continuous function $\phi$.

We also need to check that the empirical error
$\vnorm{d_n - \E d_n}_{\infty} = O(h_n^2)$ a.s..
If $nh_n^{m+4}/\log{n} \to \infty$ then 
using the Bernstein bound in equation \ref{eqn:bernstein bound}
with $\epsilon$ replaced by $h_n^2$ 
and applying Borel-Cantelli
gives the desired result.

\end{proof}

\subsection{Limit as weighted Laplace-Beltrami operator}
Under some regularity conditions, 
the 
limit given in the main theorem (Theorem \ref{thm:mainTheorem}) 
yields a
weighted Laplace-Beltrami operator.

For convenience, 
define $\gamma(x) = r_x(x)$, $\omega(x) = w_x(x)$.
\begin{corollary}
\label{cor:differentiable_r}
Assume the conditions of Theorem \ref{thm:mainTheorem} and
let $q = p^2 \omega \gamma^{m+2}$.
If $r_x(y) = r_y(x),w_x(y) = w_y(x)$ for all $x,y \in \M$ 
and $r_{(\cdot)}(\cdot),w_{(\cdot)}(\cdot)$ 
are twice differentiable in a neighborhood of $(x,x)$
for all $x$, then for $c_n' = Z_{K_0,m}/h^{m+2}$
\begin{align}
-c_n' L_{u}^{(n)} \to \frac{q}{p} \Delta_q.
\end{align}
\end{corollary}

\begin{proof}
Note that $\restrict{\nabla}_{y=x} \gamma(y) = 2 \restrict{\nabla}_{y=x} r_x(y)$.
The result follows from application of Theorem \ref{thm:mainTheorem},
Corrollary \ref{corr:Unnorm and norm Lap},
and the definition of the weighted Laplace-Beltrami operator.
\end{proof}



\section{Application to Specific Graph Constructions}

To illustrate Theorem~\ref{thm:mainTheorem}, 
we apply it to calculate the asymptotic limits of 
graph Laplacians for several widely used graph construction methods. 
We also apply the general
diffusion theory framework to analyze LLE.

\subsection{$r$-Neighborhood and Kernel Graphs}

In the case of the $r$-neighborhood graph,
the Laplacian is constructed using a kernel with fixed bandwidth 
and normalization. The base kernel is simply the indicator 
function $K_0(x) = I(|x| < r)$. 
The radius $r_x(y)$ is constant so $\dot{r}(x) = 0$. 
The drift is given by $\mu_s(x) = \nabla p(x)/p(x)$ and
the diffusion term is $\sigma_s(x)\sigma_s(x)^T = I$.
The limit operator is thus
\[\frac{1}{2}\Delta_{\M} + \frac{\nabla p(x)^T}{p(x)} \nabla 
= \frac{1}{2} \Delta_{2}
\] 
as expected. This analysis also holds for
arbitrary kernels of bounded variation. One
may also introduce the usual weight function 
$w_x^{(n)}(y) = d_n(x)^{-\alpha}d_n(y)^{-\alpha}$
to obtain limits of the form $\frac{1}{2} \Delta_{p^{2-2\alpha)}}$.
These limits match those obtained 
by \cite{HeinGraphNormalizations} and \cite{LafonThesis} 
for smooth kernels.

\subsection{Directed k-Nearest Neighbor Graph}

For kNN-graphs, the base kernel is still the 
indicator kernel, and
the weight function is constant $1$.
However, 
the bandwidth function $r^{(n)}_x(y)$ is random 
and depends on $x$.
Since the graph is directed, it 
does not depend on $y$ so $\dot{r}_x = 0$.

%
By the analysis in section \ref{sec:knn convergence},
$r_x(x) = c p^{-1/m}(x)$ for some constant $c$.
Consequently the limit operator is proportional to
\[
\frac{1}{p^{2/m}}(x) \left( 
\Delta_{\M} + 2  \frac{\nabla p^T}{p} \nabla \right) 
= \frac{1}{p^{2/m}} \Delta_{p^2}.
\]
Note that this is generally {\em not} a self-adjoint operator in $L(p)$. 
The symmetrization of the graph has a non-trivial affect to make the 
graph Laplacian self-adjoint.

\subsection{Undirected $k$-Nearest Neighbor Graph}
\label{sec:undirected knn}

We consider the OR-construction where the nodes $v_i$ and $v_j$ are linked 
if $v_i$ is a $k$-nearest neighbor of $v_j$ {\em or} vice-versa. 
In this case 
$h_n^m r^{(n)}_x(y) = \max\{\rho_n(x), \rho_n(y)\}$ 
where $\rho_n(x)$ is the distance to the $k_n^{th}$ nearest neighbor of $x$.
The limit bandwith function is non-differentiable,
$r_x(y) = \max\{p^{-1/m}(x), p^{-1/m}(y)\}$, but a Taylor-like expansion
exists with
$\dot{r}_x(x) = \frac{1}{2m}\frac{\nabla p(x)^T }{p(x)}$.
The limit operator is 
\[
\frac{1}{p^{2/m}}\Delta_{p^{1-2/m}}.
\]
which is self-adjoint in $L_2(p)$.
Surprisingly, if $m=1$ then the kNN graph construction
induces a drift {\em away} from high densiy regions.

\subsection{Conditions for kNN convergence}
\label{sec:knn convergence}

To complete the analysis, we must check the conditions 
for kNN graph constructions to
satisfy the assumptions of the main theorem. 
This is a straightforward application of existing uniform consistency
results for kNN density estimation.


Let $h_n = \left( \frac{k_n}{n} \right)^{1/m}$.
The condition we must verify is 
\begin{align*}
\sup_{y \in \M} \vnorm{r_x^{(n)} - r_x}_{\infty} = O(h_n^2) \mbox{ a.s.}
\end{align*}

We check this for the directed kNN graph, but 
analyses for other kNN graphs are similar.
The kNN density estimate of \cite{loftsgaarden65knn} is 
\begin{align}
\label{eqn:knn density}
\hat{p}_n(x) = \frac{V_m}{n (h_n r_x^{(n)}(x))^m} 
\end{align}
where $h_n r_x^{(n)}(x)$ 
is the distance to the $k^{th}$ nearest neighbor of $x$
given $n$ data points. 
Taylor expanding 
equation \ref{eqn:knn density} shows that if
$ \vnorm{\hat{p}_n - p}_{\infty} = O(h_n^2)$ a.s. then 
 the requirement on the location
dependent bandwidth for the main theorem is satisfied.

\citet{devroye77knn}'s proof for the uniform 
consistency of kNN density estimation
may be easily modified to show this. Take $\epsilon = (k_n/n)^2$
in their proof. One then sees that $h_n = k_n/n \to 0$ and
$\frac{nh_n^{m+2}}{\log{n}} = \frac{k_n^{2+2/m}}{n^{1+2/m} \log{n}} \to \infty$
are sufficient to achieve the desired bound on the error.


\subsection{``Self-Tuning'' Graphs}

The form of the kernel used in self-tuning graphs is
\[
K_n(x,y) = \exp\left( \frac{-\vnorm{x-y}^2}{\sigma_n(x) \sigma_n(y)} \right).
\]
where $\sigma_n(x) = \rho_n(x)$, the distance between
$x$ and the $k^{th}$ nearest neighbor.
The limit bandwidth  
function is $r_x(y) = \sqrt{p^{-1/m}(x)p^{-1/m}(y)}$.
Since this is twice differentiable, corollary \ref{cor:differentiable_r}
gives the asymptotic limit, which is the same as for undirected kNN graphs,
\[
p^{-2/m} \Delta_{p^{1-2/m}}.
\]

\subsection{Locally Linear Embedding}
Locally linear embedding (LLE), introduced by \citet{roweis00LLE},
has been noted to behave like (the square of) the Laplace-Beltrami 
operator \cite{BelkinLapEigenmaps}. 

Using our kernel-free framework we will show how LLE differs 
from weighted Laplace-Beltrami operators and graph Laplacians
in several ways.
1) LLE has, in general,
\emph{no well-defined asymptotic limit} without additional 
conditions on the weights. 
2) It can only behave like an {\em unweighted} Laplace-Beltrami operator.
3) It is affected by the curvature of the manifold,
and the curvature can cause LLE to 
not behave like any elliptic operator 
(including the Laplace-Beltrami operator).


The key observation is that LLE only controls for the 
drift term in the extrinsic coordinates. Thus, the diffusion term
has freedom to vary. However, if the manifold has curvature, the drift
in extrinsic coordinates constrains 
the diffusion term in normal coordinates.

The LLE matrix is defined as $(I-W)^T(I-W)$
where $W$ is a weight matrix which minimizes reconstruction error
$W = \argmin_{W'} \vnorm{(I-W')y}^2$
under the constraints $W'1 = 1$ and $W_{ij}' \neq 0$ only if $j$
is one of the $k^{th}$ nearest neighbors of $i$. Typically $k > m$
and reconstruction error $ = 0$. We will analyze the matrix
$M = I-W$.

Suppose LLE produces a sequence of matrices $M_n = I - W_n$. The row sums of 
$M_n$ are $0$. Thus, we may decompose $M_n = A_n^+ - A_n^-$
where $A_n^+,A_n^{-}$ are  generators for 
finite state Markov processes
obtained from the positive and negative weights respectively.
Assume that there is some scaling $c_n$ such that 
$c_n A_n^{+},c_n A_n^{-}$ converge to generators
of diffusion processes with drifts $\mu_+, \mu_-$ and diffusion terms
$\sigma_+\sigma_+^T,\sigma_-\sigma_-^T$. 
Set $\mu = \mu_+ - \mu_-$ and $\sigma\sigma^T = \sigma_+\sigma_+ - \sigma_-\sigma_-$.

\stitle{No well-defined limit.}
We first show there is generally no well-defined asymptotic limit
when one simply minimizes reconstruction error.
Suppose $rank(L_x) < m(m+1)/2 $ at $x$. This will necessarily be true
if the extrinsic dimension $b < m(m+1)/2 + m$.
For simplicity assume $rank(L_x)= 0$.
Minimizing the LLE reconstruction error does not constrain
the diffusion term, and $\sigma(x)\sigma(x)^T$ may be chosen arbitrarily.
Choose asymptotic diffusion $\sigma\sigma^T$ and drift $\mu$ terms that are 
Lipschitz so that a corresponding diffusion process necessarily exists. 
A diffusion with terms $2\sigma\sigma^T$ and $\mu$ will also 
exist in that case.

One may easily construct graphs for the positive and negative weights
with these asymptotic diffusion and drift terms 
by solving highly underdetermined quadratic programs.
Furthermore, in the interior of the manifold,
these graphs may be constructed so that the finite sample
drift terms are exactly equal by adding an additional constraint.
Thus, $A_n^+ \to 2 G_0 + \mu^T \nabla$
and $A_n^- \to G_0 + \mu^T \nabla$ where $G_0$ is the generator 
for a diffusion process with zero drift and diffusion term
$\sigma_-(x)\sigma_-(x)^T$.
We have $c_n M_n =  A_{n}^+ - A_{n}^- \to G_0$.
Thus, we can construct a sequence of LLE matrices that have 0 reconstruction
error but have an arbitrary limit. It is trivial to see how to modify the
construction when $0 < rank(L_x) < m(m+1)/2$.

\stitle{No drift.} 
Since $\mu_s(x) = 0$,
if the LLE matrix does behave like a Laplace-Beltrami
operator, it must behave like an unweighted one, and the density 
has no affect on the drift.

\stitle{Curvature and limit.}
We now show that the curvature of the manifold affects LLE
and that the LLE matrix may not behave like any elliptic operator.
If the manifold has sufficient curvature, namely
if the extrinsic coordinates have dimension
$b \geq m + m(m+1)/2$ and $rank(L_x) = m(m+1)/2$, then 
the diffusion term 
in the normal coordinates is fully constrained by the drift term
in the extrinsic coordinates. 

Recall from equation \ref{eqn:embedCoord} that the 
extrinsic coordinates as a function of the normal coordinates are
$y = x + H_xs + L_x(ss^T) + O(\vnorm{s}^3)$. By linearity of $H_x$ 
and $L_x$, the asymptotic drift in the extrinsic coordinates
is $\mu(x) = H_x\mu_s(x) + L_x(\sigma_s(x)\sigma_s(x)^T)$.

Since reconstruction error in the extrinsic
coordinates is 0, we have in normal coordinates 
\begin{align*}
\mu_s(x) = 0 \quad \mbox{ and } \quad L_x(\sigma_s(x)\sigma_s(x)^T) = 0.
\end{align*}
In other words, the asymptotic 
drift and diffusion terms of $A_{n}^+$ and $A_{n}^-$ must be the same,
and $c_n M_n \to G_0 - G_0 = 0$.
 
This implies that the scaling $c_n$ where LLE can be expected 
to behave like an elliptic operator gives the trivial limit 0.
If another scaling yields a non-trivial limit, it may include 
higher-order differential terms. It is easy to see 
 when $L_x$ is not full rank, the curvature affects LLE
by partially constraining the diffusion term.


\stitle{Regularization and LLE.} 
We note that while the LLE framework of minimizing reconstruction
error can yield ill-behaved solutions, practical implementations
add a regularization term when constructing the weights. 
This causes the reconstruction error to be non-zero in general
and gives unique solutions for the weights which favor equal weights 
(and asymptotic behavior like kNN graphs). 

\section{Experiments}

\begin{figure*}[t]
\begin{tabular}{cc}
    \includegraphics[width=2.0in]{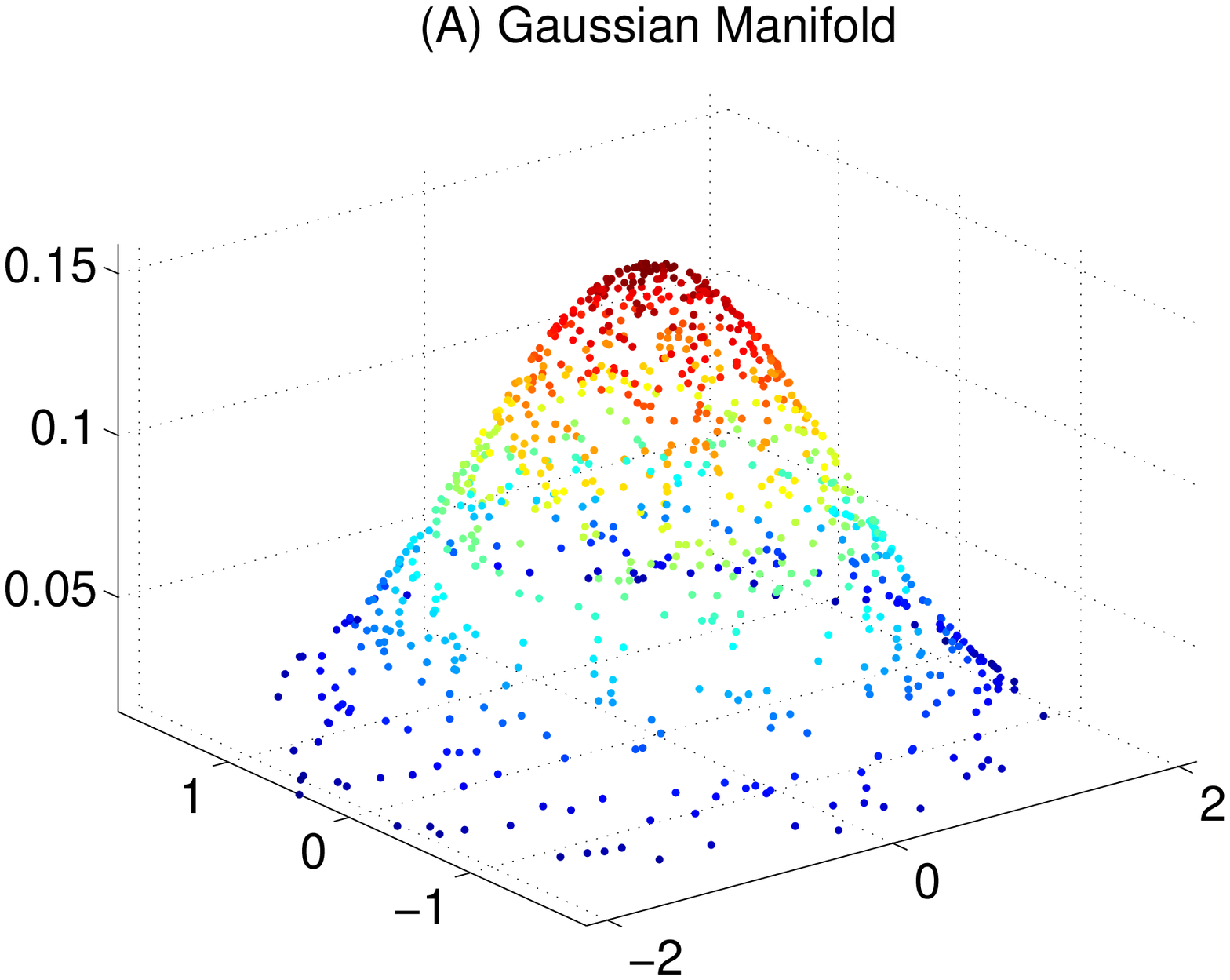}
&
    \includegraphics[width=2.0in]{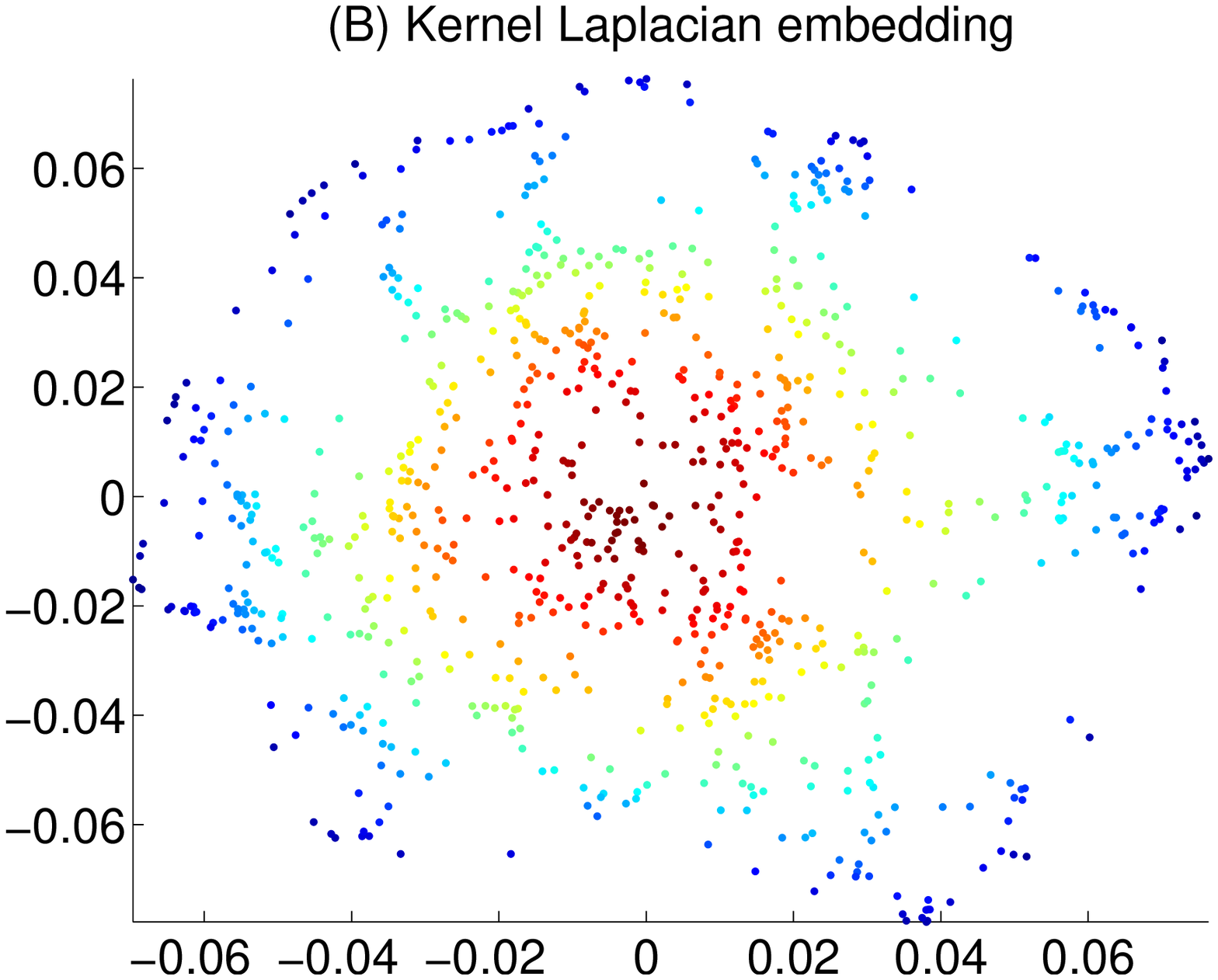} \\

    \includegraphics[width=2.0in]{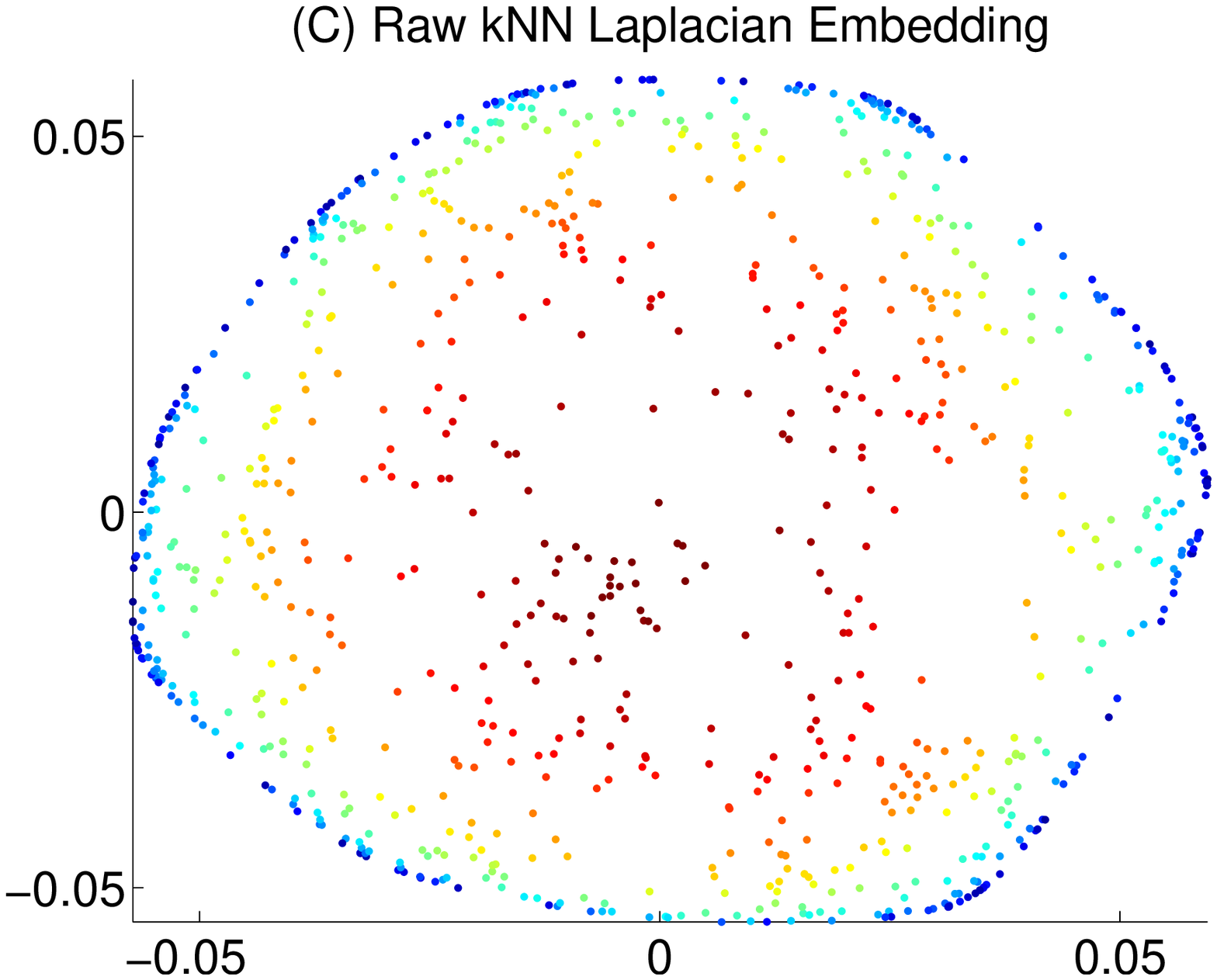}
&
    \includegraphics[width=2.0in]{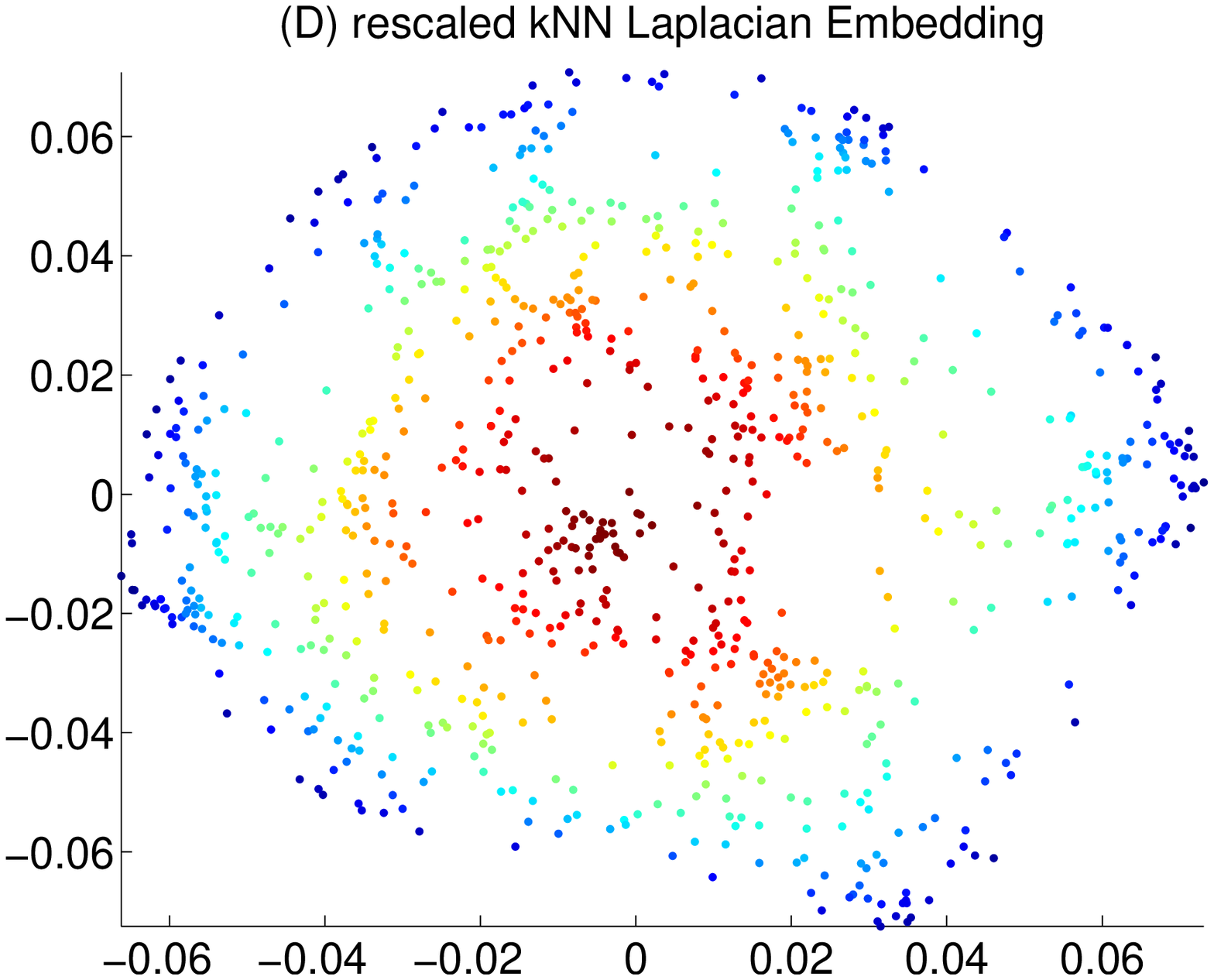}

\end{tabular}
\caption{(A) shows a 2D manifold where the $x$ and $y$ coordinates 
are drawn from a truncated standard normal distribution. 
(B-D) show embeddings using different graph constructions. 
(B) uses a normalized Gaussian kernel $\frac{K(x,y)}{d(x)^{1/2}d(y)^{1/2}}$, 
(C) uses a kNN graph, and
(D) uses a kNN graph with edge weights $\sqrt{\hat{p}(x)\hat{p}(y)}$.
The bandwidth for (B) was chosen to be the median standard deviation
from taking 1 step in the kNN graph.
 }
\label{fig:GaussMani}
\end{figure*}

To illustrate the theory, we show how to correct 
the bad behavior of the kNN Laplacian for a synthetic data set.
We also show how our analysis can predict the surprising
behavior of LLE.

\stitle{kNN Laplacian.}
We consider a non-linear embedding example which 
almost all non-linear embedding techniques handle well but
the kNN graph Laplacian performs poorly. Figure \ref{fig:GaussMani}
shows a 2D manifold embedded in 3 dimensions and  
embeddings using different graph constructions.
The theoretical limit of the 
normalized Laplacian $L_{knn}$
for a kNN graph is
$
L_{knn} = \frac{1}{p} \Delta_1.
$
while the limit for a graph with Gaussian weights
is $L_{gauss} =  \Delta_{p}$.
The first 2 coordinates of each point are 
from a truncated standard normal distribution, so the density at 
the boundary is small and the effect of the $1/p$ term is substantial. 
This yields the bad behavior shown in 
Figure \ref{fig:GaussMani} (C). 
We may use the relationship
between the $k^{th}$-nearest neighbor and the density in 
Eqn~\eqref{eqn:knn density} to obtain a pilot estimate $\hat{p}$
of the density. Choosing $w_x(y) = \sqrt{\hat{p}_n(x)\hat{p}_n(y)}$,
gives a weighted kNN graph with the same limit as the graph with 
Gaussian weights.
Figure \ref{fig:GaussMani} (D)
shows that this change yields the roughly desired behavior but with 
fewer ``holes'' in low density regions and more in high density regions.


\begin{figure*}[t]
\begin{tabular}{cc}

    \includegraphics[width=2.0in]{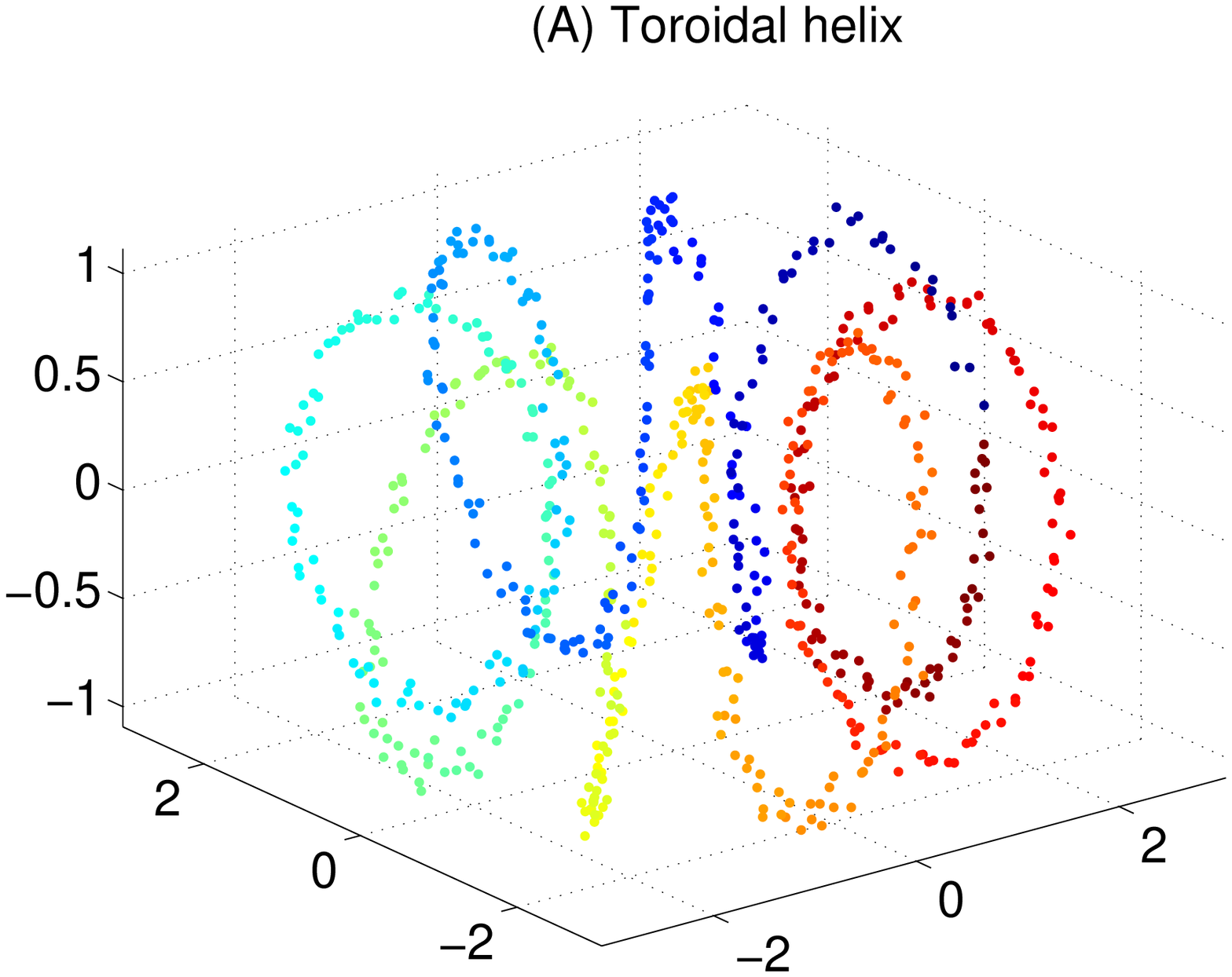}
&
    \includegraphics[width=2.0in]{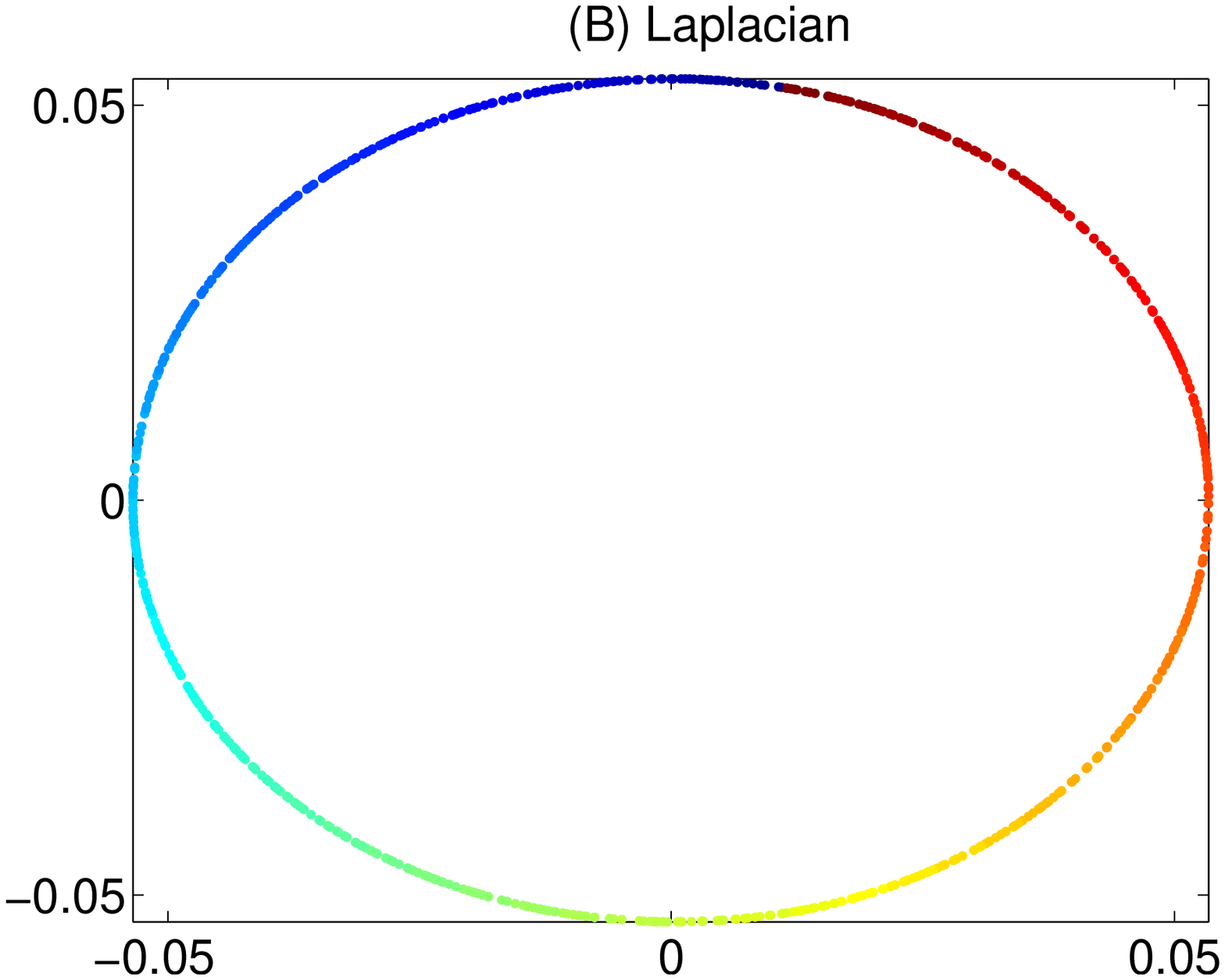} \\

    \includegraphics[width=2.0in]{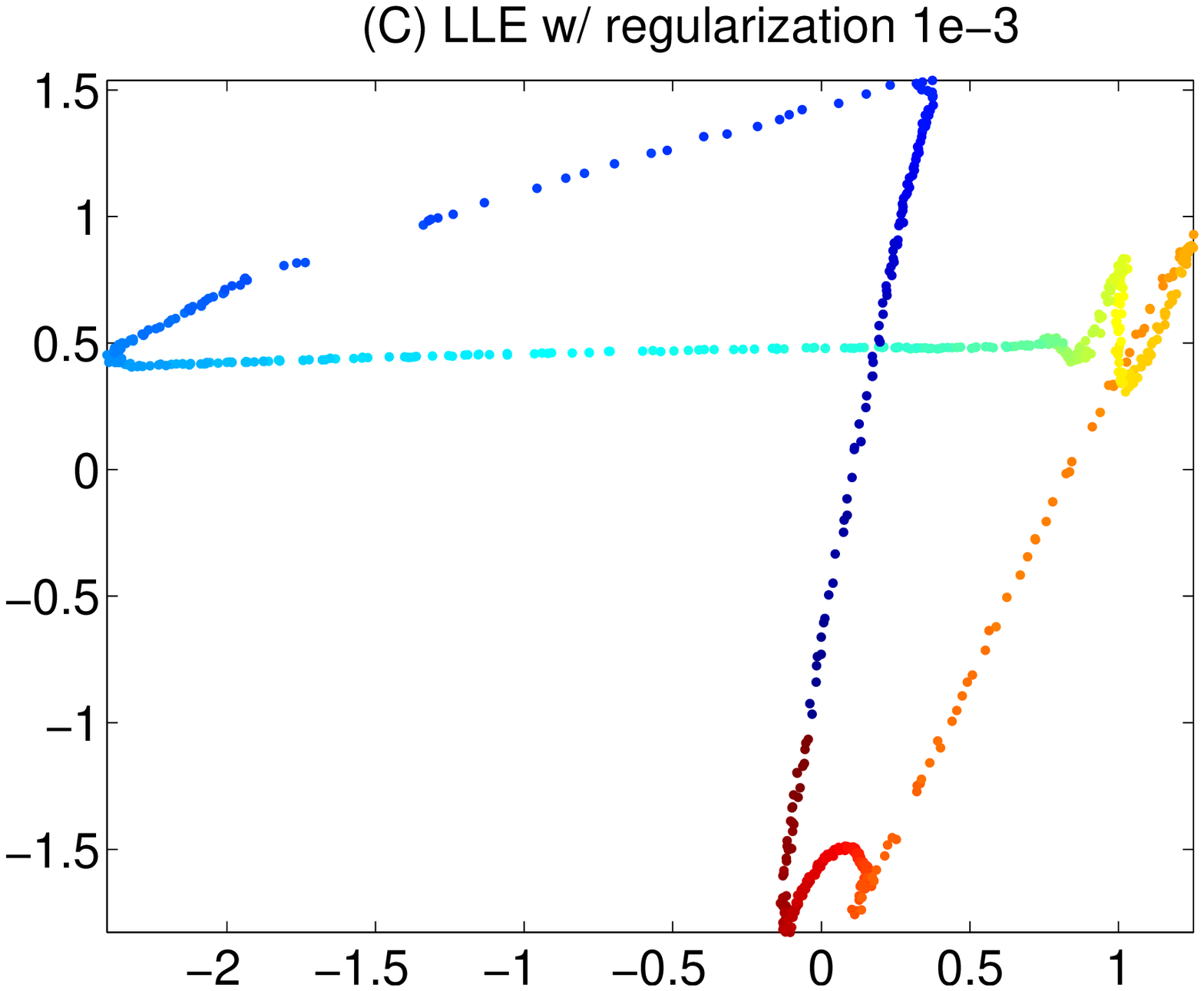}
&
    \includegraphics[width=2.0in]{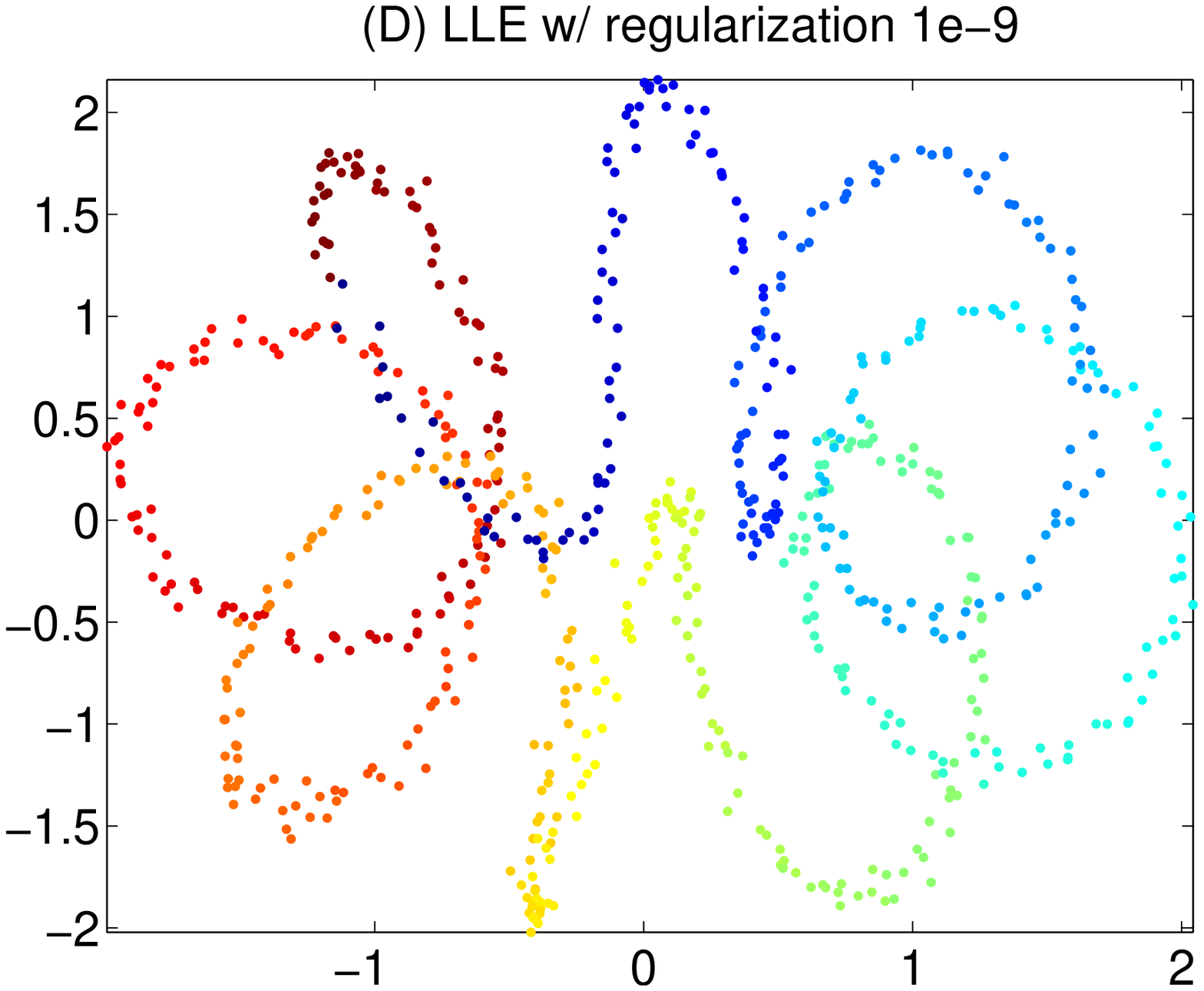}
\end{tabular}
\caption{(A) shows a 1D manifold isometric to a circle.   
(B-D) show the embeddings using  
(B) Laplacian eigenmaps which correctly identifies the structure,
(C) LLE with regularization 1e-3, and 
(D) LLE with regularization 1e-6.
 }
\label{fig:helix}
\end{figure*}

\stitle{LLE.}
We consider another synthetic data set, 
the toroidal helix,  in which the manifold structure
is easy to recover. 
Figure \ref{fig:helix} (A) shows the manifold which is clearly
isometric to a circle, a fact picked up by the kNN Laplacian in
Figure \ref{fig:helix} (B). 

Our theory predicts that the heuristic argument that 
LLE behaves like the Laplace-Beltrami operator will {\em not} hold.
Since the total dimension for the drift and diffusion terms is $2$
and the global coordinates also have dimension 2, 
that there is forced cancellation of the 
first and second order differential terms and the operator should behave
like the 0 operator or include higher order differentials. In
Figure \ref{fig:helix} (C) and (D), we see this that LLE performs poorly
and that the behavior comes closer to the 0 operator
 when the regularization term is smaller.


\section{Remarks and Discussion}

\subsection{Non-shrinking neighborhoods}
In this paper, we have presented convergence results using results for 
diffusion processes without jumps.
Graphs constructed using a fixed, non-shrinking bandwidth
do not fit within this framework, but
approximation theorems for diffusion processes with jumps 
still apply (see \cite{JacodLimitThms}). Instead of being characterized by the 
drift and diffusion pair
$\mu(x), \sigma(x)\sigma(x)^T$, the infinitesimal generators
for a diffusion process with jumps is characterized by the 
``L\^{e}vy-Khintchine'' triplet consisting of the drift, diffusion,
and ``L\^{e}vy measure.''  Given a sequence of transition kernels $K_n$,
the additional requirement for convergence of the limiting process is
the existence of a limiting transition kernel $K$ such that 
$\int K_n(\cdot,dy) g(y) dy \to \int K(\cdot,dy) g(y) dy$ locally uniformly
for all $C^1$ functions $g$. 
This establishes an impossibility result,
that no method that only assigns positive mass on 
shrinking neighborhoods can have the same graph Laplacian limit as a 
a kernel construction method where the bandwidth is fixed.

\subsection{Convergence rates}
We note that one missing element in our analysis is the derivation of 
convergence rates. For the main theorem, we note 
that it is, in fact, not necessary to apply a diffusion approximation 
theorem.
Since our theorem still uses a kernel (albeit one with much weaker
conditions), a virtually identical proof can be obtained by 
applying a function $f$ and Taylor expanding it.
Thus, we believe that similar 
convergence rates to \cite{HeinGraphNormalizations}
can be obtained. Also, while our convergence result is stated
for the strong operator topology, 
the same conditions as in Hein give weak convergence.

\subsection{Relation to density estimation}
The connection between kernel density estimation and 
graph Laplacians is obvious, namely, 
any kernel density estimation method using a non-negative kernel 
induces a random walk graph Laplacian and vice versa. 

In this paper, we have shown that
as a consequence of identifying the asymptotic degree term,
we have shown consistency of a wide class of adaptive 
kernel density estimates on a manifold.
We also have shown that on compact sets, the the bias 
term is uniformly 
bounded by a term of order $h^2$, and 
a small modification to the Bernstein bound (Eqn~\ref{eqn:bernstein bound}) 
gives that
the variance is bounded by a term of order $h^{-m}$. Both of which one 
would expect. 
This generalizes previous work on manifold density estimation
by \cite{Pelletier2005297} and \cite{SubmanifoldKDE}  to 
adaptive kernel density estimation.

The well-studied field of kernel density estimation
may also lead to insights on how to choose a good
location dependent bandwidth as well.
We compare the form of our density estimates to other
well-known adaptive kernel density estimation techniques.
The balloon estimator and sample smoothing estimators
as described by \cite{terrell1992variable} 
are respectively given by
\begin{align}
\hat{f}_1(x) &= \frac{1}{nh(x)^d}\sum_i K\left( \frac{\vnorm{x_i - x}}{h(x_i)}\right) \\
\hat{f}_2(x) &= \frac{1}{n}\sum_i \frac{1}{h(x_i)^d} K\left( \frac{\vnorm{x_i - x}}{h(x_i)}\right).
\end{align}

In the univariate case, \cite{terrell1992variable} show that
the balloon estimators yield no improvement to the asymptotic rate
of convergence over fixed bandwidth density estimates. 
The sample smoothing estimator gives a density estimate which does not necessarily
integrate to 1. However, it can exhibit better asymptotic behavior
in some cases.
The Abramson square root law estimator \citep{abramson1982bandwidth}
is an example of a sample smoothing
estimator and takes $h(x_i) = h p(x_i)^{-1/2}$. 
On compact intervals, 
this estimator has bias of order $h^4$ rather than the usual $h^2$
\citep{silvermanDensityBook}, and it achieves this bias reduction
without resorting to higher order kernels, which necessarily negative
in some region.
However, the bias in the tail for univariate Gaussian data
 is of order $(h/\log h)^2$
\citep{terrell1992variable},
 which is only marginally
better than $h^2$.

While we do not make claims of being able to reduce bias in 
the case of density estimation a manifold, in fact, we do not 
believe bias reduction to the order of $h^4$ is possible unless one makes
some use of manifold curvature information, the existing density estimation
literature suggests what potential benefits one may achieve over different
regions of a density.

\subsection{Eigenvalues/Eigenvectors}
\stitle{Fixed bandwidth case}
%
We find our location dependent bandwidth results to be of  
interest in the context of the negative result in \cite{ConsistencySC} for
unnormalized Laplacians with a fixed bandwidth. 
Their results state that for unnormalized
graph Laplacians, the eigenvectors of the discrete approximations
do not converge if the corresponding
eigenvalues lie in the range of the asymptotic degree operator $d(x)$, 
whereas for 
the normalized Laplacian, the ``degree operator'' is the identity and 
the eigenvectors converge if the corresponding eigenvalues
stay away from 1. 
Our results suggest that even with unnormalized Laplacians, one can 
obtain convergence of the eigenvectors by manipulating the range
of the degree operator through the use of a location dependent bandwidth function. 
For example, with kNN graphs we have that the degree operator is essentially 
$1$. 
For self-tuning graphs, the degree operator also converges to 1,
and since the kernels form an equicontinuous family of functions,
the theory for compact integral operators may be rigorously applied when the 
bandwidth scaling is fixed. 

Thus we can obtain unnormalized and normalized graph Laplacians that (1) have 
spectra that converges for fixed (non-decreasing) bandwidth scalings and (2)
converge to a limit that is different from that of previously analyzed normalized 
Laplacians when the bandwidth decreases to 0.

\begin{corollary}
Assume the standard assumptions. 
Further assume that for some $h_0 > 0$,
$\left\{ K_0\left(\frac{\vnorm{y-x}}{h}\right) : h > h_0 \right\}$ form 
an equicontinuous family of functions.
Let $q,g \in C^2(\M)$ be bounded
away from 0 and $\infty$. 
Set 
\begin{align}
\gamma &= \sqrt{\frac{q}{pg}} & r_x(y) &= \sqrt{\gamma(x)\gamma(y)}\\
\omega &= \left(\frac{pg}{q}\right)^{m/2}\frac{g}{p}
& w_x(y) &= \sqrt{\omega(x)\omega(y)}.
\end{align}
If  $h_n = h_1$ for all $n$, then the eigenvectors
of the normalized Laplacians converge in the sense given
in \cite{ConsistencySC}. 
If $h_n \downarrow 0$ satisfy
the assumptions of theorem \ref{thm:mainTheorem},
then the limit rescaled degree operator is $d = g$ and
\begin{align}
-c_n L_{norm}f \to g^{-1/2} \frac{q}{p} \Delta_q (g^{-1/2}f)
\end{align}
which induces the smoothness functional
\begin{align}
\ip{f,g^{-1/2} \frac{q}{p} \Delta_q (g^{-1/2}f)}_{L_2(p)} = 
\ip{\nabla (g^{-1/2}f), \nabla (g^{-1/2}f)}_{L_2(q)}.
\end{align}
\end{corollary}
\begin{proof}
Assume the $h_n \downarrow 0$ case.
Use corollary \ref{cor:differentiable_r} and solve for $\omega$
and $\gamma$ in the system of equations:
$q = p^2 \omega \gamma^{m+2}$, $g = p \omega \gamma^m$.
In the $h_n = h_1$ case, the conditions satisfy 
those given in \cite{ConsistencySC} with the modification
that the kernel is not bounded away from 0 and the additional
assumption that $p$ is bounded away from 0. Thus, the 
asymptotic degree operator $d$ is bounded away from 0,
and the proofs in \cite{ConsistencySC} remain unchanged.
\end{proof}

We note that the restriction to an equicontinuous family of kernel
functions excludes kNN graph constructions.
However, one may get around this by considering the two-step 
transition kernels $K_2(x,y) = K(x, \cdot) * K(\cdot, y)$,
where $*$ denotes the convolution operator with respect to the underlying density. 
For indicator kernels
like those used in kNN graph constructions, $K_2$ will be Lipschitz
and hence form an equicontinuous family. 
Thus, if one handles the potential issues
with the random bandwidth function, 
one may apply the theory
of compact integral operators to obtain convergence of the spectrum
and eigenvectors for kNN graph Laplacians 
when $k$ grows appropriately.

\subsection{Reasons for choosing a graph construction method}
\label{sec:construction reasons}
We highlight how our more general kernel can yield 
advantageous properties. 
In particular, it yields graphs constructions
where one can (1) control the sparsity of the 
Laplacian matrix, (2) control connectivity properties
in low density regions, (3) give asymptotic 
limits that cannot be attained using previous graph 
construction methods, and (4) give Laplacians with 
good spectral properties in the non-shrinking 
bandwidth case. 

One way to control (1) and (2) is to make the 
binary choice of using kNN or a kernel with 
uniform bandwidth to construct the graph. 
Our results show that, by using 
a pilot estimate of the density, one can obtain
sparsity and connectivity properties in 
the continuum between these
two choices.

For (3) and (4), we note that 
the limits for previously analyzed 
unnormalized 
Laplacians
were of the form $p^{\alpha-1} \Delta_{p^{\alpha}}f$.
Using corollary \ref{cor:differentiable_r}, 
one see that limits
of the form $\frac{q}{p} \Delta_q$ for any smooth, 
bounded density $q$ on the manifold can be obtained.
Equivalently, one can approximate the smoothness functional
$\vnorm{\nabla f}_{L_2(q)}^2$ for any almost any $q$, 
not just $p^\alpha$.

For normalized Laplacians, which have good spectral properties,
the previously known limits
induced smoothness functionals of the form 
$\vnorm{\nabla (p^{(1-\alpha)/2}f)}_{L_2(p^\alpha)}^2$.
With our more general kernel and any $g,q \in C^2(\M)$,
we may induce a smoothness functional of the form
$\vnorm{\nabla (gf)}_{L_2(q)}^2$. In particular,
in the interesting case where $g = 1$
and the smoothness functional is just a norm on the gradient of $f$,
i.e. $\vnorm{\nabla f}_{L_2(q)}^2$
, $q$ may be chosen to be almost
any density, not just $q=p^1$.

\section{Conclusions}
We have introduced a general framework that enables us to
analyze a wide class of graph Laplacian constructions.
Our framework reduces the problem of graph Laplacian analysis to 
the calculation of a mean and variance 
(or drift and diffusion) for any graph construction method
with positive weights and shrinking neighborhoods.
Our main theorem extends existing strong operator convergence 
results to non-smooth kernels, and introduces a general 
location-dependent bandwidth function. 
The analysis of a location-dependent 
bandwidth function, in particular, significantly extends the family 
of graph constructions for which an asymptotic limit is known. 
This family includes the previously unstudied (but commonly used) 
kNN graph constructions, unweighted $r$-neighborhood graphs, 
and ``self-tuning'' graphs.

Our results also have practical significance in graph 
constructions as they suggest graph constructions that 
(1) can produce sparser graphs than those constructed 
with the usual kernel methods, despite having the same asymptotic 
limit, and (2) in the fixed bandwidth regime, produce 
normalized Laplacians that have well-behaved 
spectra but converge to a different class of limit operators than 
previously studied normalized Laplacians.
In particular, this class of limits include those that induce
the smoothness functional $\vnorm{\nabla f}_{L_2(q)}^2$ for 
almost any density $q$.  The graph constructions may also 
(3) have better connectivity properties in low-density regions.

\section{Acknowledgements}
We would like to thank 
Martin Wainwright and Bin Yu for their helpful comments,
and our anonymous reviewers for ICML 2010 for the detailed and helpful 
review.

\bibliography{ling}
\bibliographystyle{icml2010}

\section{Appendix}

\subsection{Main lemma}

\begin{lemma}[Integration with location dependent bandwidth]
\label{lemma:IdLemma}
Let $\Id$ be the indicator function and $h >0$ be a constant.
Let $r_x$ be a location dependent bandwidth function 
that satisfies the standard assumptions,
i.e. it has a Taylor-like expansion 
\[
\td{r}_x(y) = r_x(x) + (\dot{r}_x(x) + \alpha_x \sign(u_x^Ts) u_x)^Ts 
+ \epsilon_r(x,s). 
\]
Let $V_m = \frac{\pi^{m/2}}{\Gamma\left( \frac{m}{2} + 1 \right)}$
be the volume of the unit $m$--sphere.

Then
\begin{align*}
M_0 = \frac{1}{V_m h^m} \int  \Id \left( \frac{\vnorm{y-x}}{\tilde{r}_x(s)} < h \right) ds 
 &= r_x(x)^m + h^2 \epsilon_0(x,h) &\\
M_1 = \frac{1}{V_m h^m} \int s \Id \left( \frac{\vnorm{y-x}}{\tilde{r}_x(s)} < h \right) ds 
 &= h^2 r_x(x)^{m+2} \dot{r}(x) 
+ h^3 \epsilon_1(x,h) \\
M_2 = \frac{1}{V_m h^m} \int ss^T \Id \left( \frac{\vnorm{y-x}}{\tilde{r}_x(s)} < h \right) ds 
&= \frac{2h^2}{m+2} r_x(x)^{m+2} I + 
h^3 \epsilon_2(x,h)
\\
\end{align*}
where $\sup_{x \in \M, h < h_0} \vnorm{\epsilon_i(x,h)} < C_{\epsilon}$
for some constant $C_{\epsilon} > 0$.
\end{lemma}
\begin{proof}
Let $v(s) = \dot{r}(x) + \sign(s^Tu_x) \alpha u_x$. 
We will show that the set on which the indicator function is approximately
a sphere shifted by $v/r_x(x)$ with radius $h r_x(x)$. 
\begin{align*}
\Id \left( \frac{\vnorm{y-x}}{r_x(s)} < h \right) 
 &= \Id \Big( \vnorm{s}^2 + \vnorm{L(ss^T)}^2 < h^2 ( r_x(x) + v(s)^Ts + O(\vnorm{s}^2) )^2 
 \Big) \\
 &= \Id \left( \vnorm{s}^2 < h^2 r_x(x)^2 (1 + 2 v(s)^T s + O(h^2)) \right) \\
 & = \Id \left( \vnorm{s}^2 - 2 h^2 \frac{v(s)^T s}{r_x(x)} + 
\frac{h^4 v(s)^Tv(s)}{r_x(x)^2}  < h^2 r_x(x)^2 + O(h^4) \right) \\
 & = \Id \left( \vnorm{s-\frac{v(s)}{r_x(x)}} < h r_x(x) + h^3 \delta_x(s) \right)
\end{align*}
for some function $\delta_x(s)$.
Furthermore, the assumptions on the bounded curvature of the manifold
and uniform bounds on the bandwidth function remainder term $\epsilon_r(x,s)$ 
give that the perturbation term
$\delta_x(s)$ may be uniformly 
bounded by $\sup_{x \in \M} |\delta_x(s)| \leq C_{\delta}(\vnorm{s}^2)$ 
for some constant $C_\delta$.

The result for the zeroth moment follows immediately from this.
The results for the first and second moments we calculate in
 lemma \ref{lemma:shift sphere}.
\end{proof}

\subsubsection{Refined analysis of the zeroth moment}
For convergence of the normalized Laplacian, we need a more refined
result for the zeroth moment. 

\begin{lemma}
\label{lemma:zeroth moment}
Assume 
\[
\td{r}_x(y) = r_x(s) + \epsilon_r(x,s). 
\]
where $r_x(s)$ 
is twice continuously differentiable as a function of $x$ and $s$ and 
and $\epsilon_r$ is bounded. Then
\begin{align*}
\int  \frac{1}{V_m h^m} \Id \left( \frac{\vnorm{y-x}}{\tilde{r}_x(s)} < h \right) ds 
 &= r_x(x)^m + h^2 b(x) + h^2 \epsilon_0(x,h) 
\end{align*}
where $b$ is continuous and $\sup_x | \epsilon_0(x,h) | \to 0$ as $h \to 0$.
\end{lemma}
\begin{proof}
We first sketch idea behind the proof 
and leave the details to interested readers. 
One may convert the integral in normal coordinates 
to an integral in polar coordinates $(R,\theta)$.
One may then apply the implicit function theorem to obtain that
the unperturbed radius function $R$ 
is a twice continuously differentiable function of $h$. 
This gives a Taylor expansion of the zeroth moment with respect to $h$. 
$\epsilon_r(x,s)$ gives the desired result. 


We may express the integral for the zeroth moment in polar coordinates 
$Z_x(h) = 
\int  \frac{1}{V_m h^m} \Id \left( \frac{\vnorm{y-x}}{\tilde{r}_x(s)} < h \right) ds  = \int R_x(\theta,h) d\mu_{\theta}$
where $\mu_{\theta}$ is the uniform measure on the surface of 
the unit $m$-sphere
and $\td{s} = s/h = R_x(\theta,h)) \theta$ solves the equation 
\begin{align*}
\vnorm{\td{s}}^2 + L(\td{s}\td{s}^T) &=  
\left(r_x(x) + h \nabla r_x(x)^T \td{s} + h^2 \td{s}^T \hcal_{r_x(0)}\td{s}\right)^2.
\end{align*}
and $\hcal_{r_x(0)}$ is the Hessian of $r_x(\cdot)$ evaluated at $0$.

By the implicit function theorem, the solutions $\td{s}$ define a
twice continuously differentiable function of $x, h$.
For sufficiently small $h \geq 0$,
$\td{s}$ is bounded away from $0$ since $r_x$ is bounded away from 0
and $\vnorm{s/h}$ is bounded away from $\infty$ by the bound
in lemma \ref{lemma:IdLemma}. 
Thus, $R_x(\theta,h)$ and $Z_x(h)$ are twice continuously differentiable
with bounded second derivatives.

$Z_x(h)$ then has a second-order Taylor expansion 
$Z_x(h) = Z_x(0) + Z_x'(0)h + Z_x''(0)h^2 + o(h^2)$.

By the less refined analysis in lemma \ref{lemma:IdLemma},
we have that $Z_x(0) = r_x(x)^m$ and $Z_x'(0^+) = 0$. 
One may apply a squeeze 
theorem to obtain that the contribution of the error term
$\epsilon_r(x,s)$ to the zeroth moment 
is bounded by $C_r \sup_{x,s} |\epsilon_r(x,s)|$ for some constant $C_r$,
and the result follows.
\end{proof}

\subsection{Moments of the indicator kernel / Integrating over the centered sphere in normal coordinates}

Here we calculate the first three moments of the normalized indicator kernel 
where $V_m = \int \Id(\vnorm{u} < 1) du = \int_{S_m} du$ is the volume of 
the $m$-dimensional unit sphere in Euclidean space.
\begin{lemma}[Moments for the sphere]
\label{lemma:sphereMom}
Let $K(\vnorm{s}/h) = \frac{1}{h^m V_m} \Id(\vnorm{s} < h)$.
Then the first two moments are given by:
\begin{align*}
M_0 = \int K(\vnorm{s}/h)ds = \frac{1}{h^m V_m} \int_{S_m} ds &= 1 + O(h^3)\\
M_1 = \int s K(\vnorm{s}/h)ds = \frac{1}{h^m V_m} \int_{S_m} s ds &= 0 + O(h^4)\\
M_2 = \int ss^T K(\vnorm{s}/h)ds = \frac{1}{h^m V_m} \int_{S_m} ss^T ds &= \frac{1}{m+2}\Id + O(h^4).
\end{align*}
\end{lemma}
\begin{proof}
The error terms $O(h^i)$ arise trivially
after converting normal coordinates to tangent space coordinates.
Thus, we may simply treat the integrals as integrals in $m$--dimensional
Euclidean space to obtain the leading term.
The values for $M_0$ and $M_1$ follow immediately
from the definition of the volume $V_m$
and by symmetry of the sphere.
We obtain the second moment result by 
calculating the values on the diagonal and off-diagonal.
On the off-diagonal
\begin{align*}
\frac{1}{V_m} \int_{S_m} s_i s_j ds = 0
\end{align*}
for $i \neq j$ due to symmetry of the sphere.

On the diagonal
\begin{align}
\frac{1}{V_m} \int_{S_m} s_i^2 ds 
&=
\frac{V_{m-1}}{V_m} \int_{-1}^1 s_i^2 (1-s_i^2)^{(m-1)/2} ds_i \\
&= 
 \frac{V_{m-1}}{V_m} \int_{-1}^1 s_i \times s_i (1-s_i^2)^{(m-1)/2} ds_i \\
&= 
0 +  \frac{V_{m-1}}{V_m} \int_{-1}^1 \frac{1}{m+1} (1-s_i^2)^{(m+1)/2} ds_i \\
&= 
\frac{1}{m+1} \frac{V_{m-1}}{V_m V_{m+1}} \int_{-1}^1  V_{m+1} (1-s_i^2)^{(m+1)/2} ds_i \\
&= \frac{1}{m+1} \frac{V_{m-1}}{V_{m+1}} \frac{V_{m+2}}{V_{m}} \\
\label{eqn:secMomConst}
&= \frac{1}{m+2}
\end{align}
where the last equality uses the recurrence relationship 
$V_{m+2} = \frac{2\pi}{m+2}V_{m}$.
\end{proof}


\subsection{Integrating the shifted and peturbed sphere}

Here we calculate the moments used in Lemma \ref{lemma:IdLemma}.

The integrals in lemma \ref{lemma:IdLemma} essentially involve integrating over
 sphere with (1) a shifted center $h^2 \dot{r}_x(x)$, 
(2) a symmetric shift by $\sign(s^Tu) h^2 \alpha_x u$ 
on two half-spheres, and
(3) a small perturbation 
$h^3 \delta_x(s)$. 

\begin{lemma}[Moments of the shifted and perturbed sphere]
\label{lemma:shift sphere}
Let $v_c \in \R^m$,
$u$ be a unit vector in $\R^m$,
$\beta \in \R$, and $h > 0$.
Define $\tilde{K}(s) = \Id(\vnorm{s- v_c + \sign(s^Tu) \beta u} <
h + h^3 \delta )$, so that the support of $\tilde{K}$
is a shifted and perturbed sphere with center $v_c$,
symmetric shift $\sign(s^Tu) \beta u$, and radius perturbation $h^3 \delta$.

Assume $\vnorm{v_c}, |\beta|
< C h^2 $ and $\delta < \min \{C,1\}$ for some constant $C$,
and put 
$h_{max} = h + h^3 \delta$

Then
\begin{align*}
M_0 = \frac{1}{V_m } \int_{\R^m} \tilde{K}(s) ds &= h^m + \epsilon_0 \\
M_1 = \frac{1}{V_m } \int_{\R^m} s \tilde{K}(s)ds &= h^{m+2} v_c + \epsilon_1 \\
M_2 = \frac{1}{V_m } \int_{\R^m} ss^T \tilde{K}(s)ds &= \frac{h^{m+2}}{m+2} \Id + \epsilon_2.
\end{align*}
where $\epsilon_1 < \kappa C h_{max}^{m+1}$
and $\epsilon_i < \kappa C h_{max}^{m+3}$
for $i=1,2$
and $\kappa$ is some universal constant that does not depend
on $\delta, v_c, $ or $\beta$. 
\end{lemma}

\begin{proof}
Set $H_+ = \{s \in \R^m :\, u^Ts > 0\}$ and $H_- = H_+^C$ to be the half-spaces
defined by $u$. 
For a set $H \subset \R^m$,
let $H + v_c \defn \{ w + v_c : w \in H \}$.

We first bound the error introduced by the perturbation $h^3 \delta$.
Define
\begin{align*}
\acal &\defn supp(\tilde{K}) = \lbrace s \in \R^m :  \vnorm{s- v_c + \sign(s^Tu) \beta u } < 
h + h^3 \delta \rbrace \\
\overline{\acal} &\defn \lbrace s \in \R^m :  \vnorm{s- v_c + \sign(s^Tu) \beta u} < 
h
\rbrace \\
\end{align*}
so that $\overline{\acal}$ gets rid of the dependence on the perturbation.

For any function $Q$, we have a trivial bound 
\begin{align}
\nonumber \left| \int_{\acal} Q(s) ds - \int_{\overline{\acal}} Q(s) ds \right| 
\nonumber & < Q_{max} | Vol(A) - Vol(\overline(A)) | \\
\nonumber & < Q_{max} V_m | h_{max}^m - h^m | \\
\nonumber & < Q_{max} V_m (m h_{max}^{m-1}) (h^3 \delta) \\
\label{eqn:perb err}
 & = O(h^{m+2} Q_{max}) 
\end{align}
where $Q_{max} = \sup_{\vnorm{s} < h_{max}} Q(s)$ and
$m V_{m-1}$ is the surface area of the $m$-dimensional sphere.
For $Q(s) = 1/V_m$, $s/V_m$, or $ss^T/V_m$, the corresponding $Q_{max}$ are $1/V_m$, $h_{max}/V_m$,
and $h_{max}^2/V_m$. The error induced by the perturbation
is thus of the right order.

We now consider the integral over the unperturbed but shifted sphere.
Denote by $B_h(v)$ the ball of radius $h$ centered on $v$.
Note that the function 
$\Id(s \in \overline{\acal}) = \Id( \vnorm{s -v_c + sign(s^Tu)\beta u} < h )$
is symmetric around $v_c$. 
Thus, for a function $Q(s-v_c + \beta u)$ which is symmetric around $v_c$,
\begin{align*}
\int_{\overline{\acal}} Q(s-v_c) ds 
&= 2 \int_{\overline{\acal} \cap H^+} Q(s-v_c) ds \\
&= 2 \int_{H^+} Q(s-v_c) \Id( \vnorm{s -v_c} < h ) ds - \\
& \qquad 2 \int_{H^+} Q(s-v_c) (\Id( \vnorm{s -v_c} < h ) -
\Id( \vnorm{s -v_c + \beta u} < h ) ) ds \\
&= \int Q(s) \Id( \vnorm{s} < h ) ds - \\
& \qquad 2 \int_{H^+} Q(s-v_c) (\Id( s \in B_h(v_c)) -
\Id( s \in B_h(v_c - \beta u) ) ) ds.
\end{align*}
For $Q(s) = 1/V_m$ or $ss^T/V_m$,
lemma \ref{lemma:sphereMom} gives that 
the value of the main term $\int Q(s) \Id( \vnorm{s} < h ) ds$ is
$h^m$ or $\frac{h^{m+2}}{m+2} I$ respectively.
The error term is bounded by
\begin{align*}
& 2 \int_{H^+} Q(s-v_c) (\Id( s \in B_h(v_c)) -
\Id( s \in B_h(v_c - \beta u) ) ) ds \\
& \qquad \leq 2 Q_{max} \int_{H^+} \abs{\Id( s \in B_h(v_c)) -
\Id( s \in B_h(v_c - \beta u) ) } ds \\
& \qquad < 2 Q_{max} \abs{\beta} Area( H^{+} \cap B_h(v_c)) \\
& \qquad < 2 Q_{max} \abs{\beta} (m V_{m-1} h^{m-1}) \\
& \qquad < 2 m V_{m-1} C Q_{max} h^{m+1}
\end{align*}
where $Area( H^{+} \cap B_h(v_c))$ is the 
surface area of a half-sphere
of radius $h$.
Plugging in $Q_{max} = 1/V_m$ and $h^2/V_m$ give that the error terms
for the zeroth and second moment calculations are of the right order.

By another symmetry argument, we have for the first moment calculation
$\int_{\overline{\acal}} \frac{1}{V_m} (s-v_c) ds = 0$
or equivalently,
\begin{align*}
\frac{1}{V_m} \int_{\overline{\acal}} s ds &= \frac{v_c}{V_m}  \int_{\overline{\acal}} ds \\
\label{eq:symm shift 1}
&= h^{m} v_c + O(h^{m+3})
\end{align*}
where the last equality holds from the calculation of the zeroth moment above.
More precisely, the error term is bounded by $2 m V_{m-1} C Q_{max} h^{m+1} v_c$.

\end{proof}

\end{document}